\documentclass[preprint,12pt]{elsarticle}

\usepackage[margin=1in]{geometry}
\usepackage[x11names]{xcolor}
\usepackage{array}
\usepackage{tabularx}
\usepackage{booktabs}

\usepackage{hyperref}
\usepackage{amsfonts} 
\usepackage{breqn}
\usepackage{multirow}
\usepackage{pdflscape}
\usepackage{float}
\usepackage{caption}
\usepackage[dvipsnames]{xcolor}
\usepackage{bm} % For \bm
\definecolor{darkred}{rgb}{0.9,0.1,0.1}

\usepackage{amsthm}
\newtheorem{theorem}{Theorem}[section]

\newtheorem{lemma}[theorem]{Lemma}

\newtheorem{proposition}[theorem]{Proposition}
\theoremstyle{definition}

\newtheorem{remark}[theorem]{Remark}

\usepackage{mathtools}
\usepackage{algorithmic}

\newcommand{\R}{\mathbb{R}}

\usepackage[ruled,vlined]{algorithm2e}
\hypersetup{
 colorlinks,
 citecolor=black,
 filecolor=black,
 linkcolor=black,
 urlcolor=black
}
\usepackage{amssymb}

\journal{arXiv}

\begin{document}

\begin{frontmatter}

\title{A Variational Framework for Residual-Based Adaptivity in Neural PDE Solvers and Operator Learning}

\author[inst1,label1]{Juan Diego Toscano}

\author[inst1,label1]{Daniel T.~Chen}
\author[inst2]{Vivek Oommen}
\author[inst1]{\\Jérôme Darbon}
\author[inst1,label2]{George Em Karniadakis}
\affiliation[inst1]{organization={Division of Applied Mathematics, Brown University},%Department and Organization
%   addressline={}, 
   city={Providence},
   postcode={02912}, 
   state={RI},
   country={USA}}
\affiliation[inst2]{organization={School of Engineering, Brown University},%Department and Organization
%   addressline={}, 
   city={Providence},
   postcode={02912}, 
   state={RI},
   country={USA}}

\fntext[label1]{These authors contributed equally to this work}
\fntext[label2]{Corresponding author: george\_karniadakis@brown.edu}
\begin{abstract}
%150 words
Residual-based adaptive strategies are widely used in scientific machine learning but remain largely heuristic. We introduce a unifying variational framework that formalizes these methods by integrating convex transformations of the residual. Different transformations correspond to distinct objective functionals: exponential weights target the minimization of uniform error, while linear weights recover the minimization of quadratic error. Within this perspective, adaptive weighting is equivalent to selecting sampling distributions that optimize the primal objective, thereby linking discretization choices directly to error metrics. This principled approach yields three benefits: (1) it enables systematic design of adaptive schemes across norms, (2) reduces discretization error through variance reduction of the loss estimator, and (3) enhances learning dynamics by improving the gradient signal-to-noise ratio. Extending the framework to operator learning, we demonstrate substantial performance gains across optimizers and architectures. Our results provide a theoretical justification of residual-based adaptivity and establish a foundation for principled discretization and training strategies.

\end{abstract}

% %%Graphical abstract
% \begin{graphicalabstract}
% \includegraphics{grabs}
% \end{graphicalabstract}


\end{frontmatter}

%\tableofcontents

%% main text

\section{Introduction}
\label{Introduction}

Scientific machine learning (SciML) has emerged as a powerful alternative to traditional numerical methods for solving partial differential equations (PDEs). Here, we consider two of the main approaches in SciML. The first, which includes Physics-Informed Neural Networks (PINNs)~\citep{raissi2019deep} and their variants~\citep{shukla2020physics}, focuses on function approximation, where a representation model is trained to satisfy the governing equations of a specific problem~\cite{toscano2025pinns}. The second is operator learning~\citep{lu2019deeponet,li2020fourier}, where a model learns the underlying solution operator itself, allowing it to generate solutions for new boundary conditions, source terms, or parameters almost instantaneously.

At their core, SciML models employ parameterized functions with strong approximation capabilities~\citep{liu2024kan,liu2024kan2,wang2024expressiveness,toscano2025kkans,stenkin2024mathematical,uddin2023wavelets,wu2023gpt,song2025explicit} to represent the solution of a PDE. The problem is thus transformed into an optimization task to find the optimal parameters for this representation.

For physics-informed methods, this typically involves minimizing a loss function composed of the PDE residuals and the mismatch with observational data.
This optimization-centric approach provides significant flexibility over traditional methods. For instance, PINN-style methods can easily incorporate sparse data and are not constrained by prescribed boundary conditions, which makes them highly effective for solving inverse problems. Furthermore, they can be scaled to high-dimensional problems, offering a way to mitigate the so-called curse of dimensionality~\citep{hu2024tackling}.
However, the optimization problems inherent to SciML are generally high-dimensional and non-convex, making models susceptible to converging to poor local minima. Consequently, tackling the optimization has drawn significant research attention, with efforts including the development of specialized optimizers~\citep{jnini2024gauss,urban2024unveiling,kiyani2025optimizing} and methods that simplify the optimization task by explicitly encoding physical constraints, such as the exact imposition of boundary conditions~\citep{zeinhofer2024unified,sukumar2022exact}.

Among the various approaches, one of the most prominent strategies, which addresses both optimization and {discretization} errors, is to modify the loss function itself. This is typically achieved through adaptive sampling~\citep{lu2021deepxde,wu2023comprehensive} and weighting schemes~\cite{mcclenny2023self,anagnostopoulos2024residual,chen2025self,basir2023adaptive,basir2022physics, hu2025conditionally}. Instead of altering the model or the PDE, these techniques dynamically adjust the training process to focus on regions of the domain that are more difficult to learn. Adaptive sampling methods achieve this by concentrating collocation points in areas where the PDE residual is high~\citep{wu2023comprehensive, daw2022rethinking, gao2023active}, while adaptive weighting methods assign larger local weights to these same important regions. The strategies for determining these weights are diverse, ranging from direct residual-based schemes to more complex adversarial or augmented Lagrangian formulations~\citep{anagnostopoulos2024residual,mcclenny2023self,basir2022physics, son2023enhanced}.

Due to their simplicity and efficiency, methods based on the residuals are particularly popular, as they do not require specialized architectures or additional parameters. Two such examples are residual-based attention (RBA), which applies adaptive weights, and the residual-based adaptive distribution (RAD), which modifies the sampling of collocation points~\citep{toscano2024inferring_AIV,ramireza2024residual,wang2024aspinn,ramirez2025residual,wang2024general,chen2024self,rigas2024adaptive,wu2025fmenets,si2025convolution,toscano2025mr}. Although these methods intuitively aim for the same goal of directing the optimizer's attention to high-error regions, no direct link between them has been established. Furthermore, while focusing on high-error regions seems beneficial, a formal argument for its efficiency has been lacking. These heuristic strategies are conceptually related to importance sampling; however, there is a critical difference. Standard importance sampling reweights the sample to produce an unbiased estimator with less variance, whereas the schemes used in SciML estimate the desired functionals under an adaptively biased distribution. To the best of our knowledge, no theoretical understanding of this adaptive biasing in SciML exists.

In this work, we address this gap by proposing a general framework for deriving these sampling and weighting schemes. The formulation is most clearly seen by changing the objective function from the usual mean-squared error ($L^2$) to the maximum error ($L^\infty$) over the spatial domain. Leveraging variational formulas, we show that minimizing the $L^\infty$-norm can be written in a dual form that naturally involves sampling adaptively from distributions exponentially tilted by the current residual. These new distributions can be realized through either direct sampling or importance weights, thereby providing a formal justification and unification for these training schemes. More general loss functions, such as the $L^2$ norm, can be recovered from this dual formulation using variational representations of more general statistical divergences. We refer to the multipliers obtained from this variational approach as variational residual-based attention (vRBA).

Our unified framework provides a principled origin for heuristics like RBA and RAD, showing how different potential functions correspond to different adaptive schemes. We extend these methods to operator learning with a hybrid strategy employing importance sampling over the function space and importance weighting over the spatial domain, which can be seamlessly integrated into architectures like FNO and DeepONet. The framework yields a twofold benefit: it lowers the discretization error by reducing the variance of the loss estimator, and it improves learning dynamics by enhancing the signal-to-noise ratio of the gradients, leading to faster convergence.%
Finally, we demonstrate the efficacy of vRBA across a range of challenging PINN and operator learning tasks. Our empirical results show that using vRBA is critical to achieve lower errors, providing significant improvements even when paired with state-of-the-art second-order optimizers~\citep{urban2024unveiling} or specialized architectures like TC-UNet~\citep{ovadia2025real}.

% The remainder of this paper is organized as follows. In Section~\ref{problen_steup}, we describe the problem formulation and the framework for examining learning dynamics. In Section~\ref{vRBA}, we present our problem reformulation and introduce the vRBA method. In Section~\ref{sec:methods}, we detail the algorithm, and in Section~\ref{Results}, we present the numerical results. Finally, we conclude with a summary of our findings in Section~\ref{summary}.

\section{Problem Setup}
\label{problen_steup}
For a domain $\Omega \subset \R^d$, define the residual $r: \Omega \to \R_+$ to be a bounded function describing the error at each spatial point $x \in \Omega$.
For example, in supervised learning tasks such as function approximation, one seeks to approximate some function $\bar u: \Omega \to \R$ with a parameterized $u(x;\theta)$, for parameters $\theta$ in some parameter space $\mathcal T$.
Then, the residual is defined as follows 
\begin{align}
\label{F_fit_loss}
    r(x) \coloneqq | \bar u(x)-u(x; \theta) |.
\end{align}
On the other hand, PINNs, the residual is given by an appropriate differential operator $\mathcal F$ applied onto the parameterized function which reads
\begin{align}
\label{PDE_Loss}
    r(x) \coloneqq | \mathcal F[u(\cdot; \theta)] |.
\end{align}
In either case, one would like to find the parameter $\theta^*$ that minimizes the residual, to approximate the unknown function $\hat u$ or to solve the differential equation $\mathcal F[u] = 0$, by defining a loss function $\mathcal L$ in terms of the residuals and solving for 
\begin{align*}
    \theta^* = \arg \min_{\theta \in \mathcal T} \mathcal L(\theta, \Omega).
\end{align*}
There are many possible choices of loss functions.
In the ideal scenario, one would like to find the parameter that minimizes the residual \textit{uniformly} over all points in the spatial domain, i.e., we seek the solution to the following optimization problem
\begin{align}
    \min_{\theta} \left\{ \max_{x\in \Omega } r(x,\theta) \right\}.
    \label{eq:p1}
    \tag{P1}
\end{align}
In other words, we wish to minimize the $L^\infty(\Omega)$-norm of the residual.
Alternatively, we can (superfluously) write \eqref{eq:p1} as an optimization over the space of probability measures
\begin{align}
    \min_{\theta}\left\{ \max_{q\in \mathcal P(\Omega) } \int_\Omega r(x,\theta) q(dx) \right\} 
    \label{eq:p1-1}
\end{align}
where the optimizer of the inner maximum is $q^* = \delta_{x^*}$ and $x^* = \arg \max_\Omega r$.

\eqref{eq:p1} is rarely used in practice due to its instability and non-differentiability.
Moreover, the optimization landscape induced by the parameterization is extremely non-convex and identically-zero residual is impossible for most parameterization schemes, e.g., neural networks, despite the availability of asymptotic approximation theorems \cite{Barron1993universal, Park1991universal, Hornik1991approximation}.
Consequently, the popular alternative is to solve a ``weaker" problem by solving
\begin{align}
    \min_\theta \left\{ \int_{\Omega} r(x,\theta)^2 p(x)dx \right\} \quad\text{where}~~p(x) = \frac{\mathbf{1}_{x\in \Omega}}{|\Omega|} 
    \label{eq:p2}
    \tag{P2}
\end{align}
and $|\Omega|$ refers to the volume of the domain.
This corresponds to the $L^2(\Omega)$ norm of the residuals and \eqref{eq:p1} is stronger than \eqref{eq:p2} in the sense that any near-optimizers of \eqref{eq:p1} are also near-optimizers in \eqref{eq:p2}.

\begin{remark}
In the context of PINNs, one's goal is to find the solution to certain partial differential equation.
Then, solutions where \eqref{eq:p1} are zero correspond to classical solutions of differential equations, i.e., $k$-times continuously-differentiable functions $u \in \mathcal C^k(\Omega)$ for some $k \geq 1$ corresponding to the highest-order derivatives in the respective differential operator.
Under this interpretation, one can potentially argue that \eqref{eq:p1} is a more natural optimization problem to consider than \eqref{eq:p2}.
\end{remark}
\noindent
Under \eqref{eq:p2}, we define the loss function
\begin{align}
\label{true_obj}
    \mathcal{L}(\theta) = \int_{\Omega} r(x,\theta)^2 p(x)dx,
\end{align}
which can be easily approximated via Monte Carlo integration, that is, let $(X_i)_{i=1}^N$ be independent samples from $p$, and define the discrete loss
\begin{align}
\label{discrete_loss2}
    \mathcal{L}^{(D)}(\theta)=\frac{1}{N}\sum_{i=1}^N r^2(X_i,\theta).
\end{align}
The discrete domain $D = (X_i)_{i=1}^N$ is also called the quadrature or the collocation points.
Minimizing this discrete loss is equivalent to minimizing the empirical mean-squared error (MSE) of a set of points randomly selected from the domain of interest. 

\subsection{Reformulating loss function for adaptive training}

The discrete loss formulation (e.g., Equation~\eqref{discrete_loss2}) has been observed to cause difficulties for certain problems. 
For instance, in PINNs for solving complex PDEs  (e.g., Allen-Cahn or Burgers' equations), it can lead to slow convergence or convergence to an incorrect solution \cite{raissi2019deep}. 
In response, two primary families of approaches have been proposed to modify the loss computation at each step of the training process.

One family of methods involves sampling the quadrature points according to a suitable residual-based distribution~\cite{lu2021deepxde, wu2023comprehensive, daw2022rethinking, gao2023active, tang2021deep, peng2022rang, zeng2022adaptive, hanna2022residual, subramanian2023adaptive, nabian2021efficient, zapf2022investigating, toscano2025aivt, daw2022mitigating}. 
A widely adopted example is the residual-based adaptive distribution (RAD)~\cite{wu2023comprehensive}, where instead of sampling the quadrature points from the base distribution $p$, one would sample from an adaptive, tilted distribution computed as follows
\begin{align}
\label{q_RAD}
    q_{\text{RAD}}^k(x) \propto \frac{r(x)^\nu}{\mathbb{E}_p[r(x)^\nu]} + c
\end{align}
where $\nu$ and $c$ are hyperparameters and $k$ is the index of the current iteration.

A second line of work, with seminal contributions from~\cite{mcclenny2023self}, proposes modifying Equation~\ref{discrete_loss2} by assigning local multipliers $\lambda_i$ to each collocation point $X_i$ and computing a weighted sum instead. 
More explicitly, one can compute an alternative loss function of the form
\begin{align*}
    \mathcal{L}_W^{(D)} (\theta)=\frac{1}{N}\sum_{i=1}^N[\lambda_i r(X_i,\theta)]^2,
\end{align*}
while numerous proposals have been made for the specific form of $\lambda_i$~\cite{mcclenny2023self,zhang2023dasa,basir2022physics,basir2022investigating,basir2023adaptive,anagnostopoulos2024residual,song2024loss,shukla2024comprehensive,ramirez2024residual,chen2024self}.
One of the most effective methods define the multipliers explicitly from the residual, as in the residual-based-attention (RBA) method: for the spatial point $X_i$, we set
\begin{align}
\label{q_RBA}
    \lambda_i^k \propto \frac{r(X_i)}{\sum_j r(X_i)}.
\end{align}
In practice, for stability, the local multipliers computation often involves an exponential moving average of the form ${\lambda}_i^k = \gamma {\lambda}^{k-1}_i + \frac{\eta}{\max_{\Omega} q} q_i^k$. 
% Notice that since the residuals $r_i$ are bounded, $\max_{\Omega} q$ is well defined.

Both approaches, therefore, follow a general theme: they modify the objective in Equation~\ref{discrete_loss2} to bias the optimization towards high-error regions. These methods have been widely adopted and have proven quite successful; however, their theoretical foundation remains elusive.

\subsection{Learning Dynamics}
\label{sec:learning-dynamics}

Once the objective function ($\mathcal L$) is calculated, the model parameters are generally optimized using a line search algorithm of the form
\begin{align}
    \label{eq:line-search}
    \theta^{k+1} &= \theta^{k}+\alpha^{k} p^{k}
\end{align}
where $\alpha^{k}$ is the step size, and $p^{k}=-H_k\nabla_{\theta}\mathcal{L}(\theta^{k})$ is the update direction which depends on the gradient of the objective function and some symmetric matrix $H_k$ \cite{urban2024unveiling}. For first-order optimizers such as Adam~\cite{kingma2014adam}, $H_k$ is treated as the identity matrix. Therefore, for the discrete case, the update direction induced by Equation~\eqref{discrete_loss2} is given by
\begin{align}
    \label{update_discrete}
    \hat p_k \coloneqq - \nabla_\theta L^{(D)}(\theta) &= - \frac{1}{N} \sum_{i=1}^N \nabla_\theta r(X_i, \theta^k)^2. 
\end{align}
Notice that a successful optimization would be achieved if the discrete update direction contains enough information to minimize the loss in the whole domain. However, even if the sampling is performed only once at the beginning of training and the full data is used to perform the update stage (commonly referred to as full-batch training), the optimization process can lead to imbalances.

\begin{lemma}
For a discrete domain $D$ of size $N = Mb$, define $\{B_j\}_{j=1}^b$ to be an equal-sized partition of $D$, that is 
\begin{align*}
    \bigcup_{j=1}^b B_j = D, ~~ B_j \cap B_{j'} = \emptyset ~~\text{for}~~j \neq j', ~~ \text{and}~~|B_j| = M~~\text{for all}~~ j = 1, \dots, b.
\end{align*}
Letting $\hat p_j^k = -\nabla_\theta \mathcal{L}^{(B_j)}(\theta^k)$ be the update direction using the batch $B_j$, the full-batch update direction is the average of the update directions of its partitions
\begin{align*}
    \hat p^k = \frac{1}{b} \sum_{j=1}^b \hat p_j^k
\end{align*}
\end{lemma} 

\begin{proof}
The proof follows from the linearity of the gradient operator. The loss on the full discrete domain, Equation~\eqref{discrete_loss2}, can be expressed as the average of the losses on the partitions $\{B_j\}$. 
\begin{align*}
 \mathcal{L}^{(D)}(\theta) &= \frac{1}{N} \sum_{i=1}^N r^2(X_i, \theta^k) = \frac{1}{bM} \sum_{j=1}^b \sum_{x \in B_j} r^2(x, \theta) \\
 &= \frac{1}{b} \sum_{j=1}^b \left( \frac{1}{M} \sum_{x \in B_j} r^2(x, \theta) \right) \\
 &= \frac{1}{b}\sum_{j=1}^b \mathcal{L}^{(B_j)}(\theta)
\end{align*}
Applying the negative gradient operator $-\nabla_\theta$ to both sides yields the result.
% \begin{align*}
%  -\nabla_{\theta}\mathcal{L}^{(D)}(D,\theta^{k}) &= -\nabla_{\theta} \left( \frac{1}{b}\sum_{i=1}^b \mathcal{L}^{(D)}(B_j,\theta^{k}) \right) \\
%  p^{k} &= \frac{1}{b}\sum_{i=1}^b \left( -\nabla_{\theta} \mathcal{L}^{(D)}(B_j,\theta^{k}) \right) \\
%  p^{k} &= \frac{1}{b}\sum_{i=1}^b \hat p_j^{k} \quad \qedhere
% \end{align*}
\end{proof}

The arbitrary partitioning of the dataset highlights a critical challenge. Even when considering the full dataset (i.e., a single computation or full batch), the gradients implicitly computed from different data partitions may point in conflicting directions. If these component gradients are not in alignment, their average, which constitutes the full-batch gradient, can be diminished or misleading, leading to slow convergence or stagnation.
\subsection{Signal-to-Noise Ratio (SNR)}

As shown in the preceding section, gradient disagreement is not merely an artifact of stochastic sampling; it is an inherent property of the loss landscape, observable even within a single, complete batch of data. To analyze these variations, several studies~\cite{schaul2013no,anagnostopoulos2024learning,tishby2000information} have considered the signal-to-noise ratio (SNR). A common formulation for the SNR of a gradient vector reads
\begin{align}
    \label{SNR}
    \text{SNR} = \frac{\|\mathbb{E}_{B_j}[ \hat p_j^{k}]\|_2}{\sqrt{\text{Tr}(\text{Var}_{B_j}[ \hat p_j^{k}])}}
\end{align}
where the expectation $\mathbb{E}_{B_j}$ and variance $\text{Var}_{B_j}$ are taken over the random selection of a partition $B_j$ from the full dataset.

To formalize this, we can decompose the update direction $\hat p_j^{k}$ from a partition $B_j$ relative to the true update direction from the continuous loss in Equation~\eqref{true_obj}, which we define as our ``signal": $\hat{p}^{k}=-\nabla_\theta \mathcal{L}^{(D)}(\theta^k)$. The partition-based direction is 
\begin{align*}
    \hat p_j^{k} = \hat p^{k} + \epsilon_i^{k}, 
\end{align*}
where $\epsilon_i^{k}$ is the deviation induced by the discretization and the specific choice of partition $B_j$. 
The discrete partition gradients are unbiased in the sense that $\mathbb{E}_{B_j}[\hat p_j^{k}] = \hat{p}^{k}$, it follows that the noise term has zero mean:
\begin{align*}
    \mathbb{E}_{B_j}[\varepsilon_i^{k}] = \mathbb{E}_{B_j}[\hat p_j^{k} - \hat{p}^{k}] = \mathbb{E}_{B_j}[\hat p_j^{k}] - \hat{p}^{k} = 0.
\end{align*}
On the other hand,  $\hat p_j^k$ has covariance
\begin{align*}
    \text{Var}_{B_j}[\hat p_j^{k}] &= \mathbb{E}_{B_j}\left[ (\hat p_j^{k} - \mathbb{E}_{B_j}[\hat p_j^{k}]) ( \hat p_j^{k} - \mathbb{E}_{B_j}[\hat p_j^{k}])^T \right] \\
    &= \mathbb{E}_{B_j}\left[ (\hat{p}^{k} + \varepsilon_i^{k} - \hat{p}^{k}) (\hat{p}^{k} + \varepsilon_i^{k} - \hat{p}^{k})^T \right] \\
    &= \mathbb{E}_{B_j}[\varepsilon_i^{k}(\varepsilon_i^{k})^T].
\end{align*}
The denominator in the SNR formula corresponds to the noise power, which is the trace of this covariance matrix
\begin{align*}
    \text{Tr}(\text{Var}_{B_j}[\hat p_j^{k}]) = \mathbb{E}_{B_j}[\text{Tr}(\varepsilon_i^{k}(\varepsilon_i^{k})^T)] = \mathbb{E}_{B_j}[\|\varepsilon_i^{k}\|_2^2].
\end{align*}

Therefore, the SNR is given by the ratio of the magnitude of the signal (the true update direction) to the root mean square (RMS) magnitude of the noise (the deviations due to partitioning).
% \begin{align*}
%     \text{SNR} = \frac{\|\hat{p}^{k}\|_2}{\sqrt{\mathbb{E}_{B_j}[\|\varepsilon_i^{k}\|_2^2]}}
% \end{align*}
This phenomenon of quantifying the disagreement, or ``noise",  between gradients has been analyzed extensively in the context of stochastic and minibatch training~\cite{schaul2013no}. Similar metrics have also been studied for different problem settings, such as multiobjective loss functions~\cite{wang2025gradient}.

\subsection{Stages of Learning}

The SNR provides insight into the deterministic and stochastic regimes of the training process, where high SNR corresponds to confident, deterministic updates, and low SNR indicates a more exploratory, stochastic phase~\cite{shwartz2017opening}. By tracking the evolution of the SNR, recent studies have identified three distinct stages of learning: fitting, transition, and diffusion. This phased progression has been observed across a variety of domains, including function approximation~\cite{shwartz2017opening,toscano2025kkans}, PINNs~\cite{anagnostopoulos2024learning,shukla2024comprehensive}, and operator learning~\cite{toscano2025kkans}, for diverse representation models such as MLPs and KANs. The three stages can be described as follows:

\paragraph{Fitting} This is the initial phase of training. The process begins with large gradients and high agreement between subdomain updates, yielding a high SNR. During this deterministic phase, the model rapidly reduces the training error by learning the dominant trends in the data. As these general features are captured, however, disagreements between subdomains on finer details emerge, causing the SNR to decrease. This stage is thus characterized by a sharp initial reduction in training error with minimal improvement in generalization error, accompanied by an increase in the model's geometric complexity~\cite{shwartz2017opening}.

\paragraph{Transition} Following the initial fit, the model enters an exploratory stage characterized by a sustained low SNR. Here, the model attempts to reconcile conflicting objectives from different subdomains, but there is no consensus on an optimal update direction. Consequently, this phase exhibits little to no improvement in the generalization error, while the geometric complexity may continue to increase as the model explores the solution space~\cite{anagnostopoulos2024learning,shukla2024comprehensive}.

\paragraph{Diffusion} After a period of exploration, the model may become sufficiently complex to find a parameter configuration that realigns the subdomain gradients. This leads to a sudden increase in the SNR, marking the beginning of the productive diffusion phase. Once this consensus direction is found, the generalization error improves significantly. Concurrently, the geometric complexity often decreases as the model identifies an efficient internal representation. Eventually, as the loss converges and the gradient magnitudes (the signal) diminish, the SNR naturally decays again, at which point learning slows down and the generalization error plateaus.

One of the main advantages of this approach is that it provides a measure of convergence that can be used to evaluate a model's performance. Previous studies~\citep{anagnostopoulos2024learning, shukla2024comprehensive} have shown that the best-performing models typically reach the total diffusion phase earlier than others;  moreover, in general, they maintain a consistently higher SNR. Conversely, models that fail to converge, often remain trapped in the diffusion stage, unable to progress further \cite{toscano2025pinns}.

\section{A framework for approximate uniform minimization}
\label{vRBA}
In this work, we propose a dual formulation of \eqref{eq:p1} inspired by the following problem
\begin{align*}
    \min_{\theta \in \mathcal T} \max_{q \in \mathcal P(\Omega)} \left\{ \int_\Omega r(x;\theta) q(dx) + \epsilon \mathbf H(q | p) \right\} \quad \text{for}~ 0 < \epsilon \ll 1
\end{align*}
where $\mathbf H$ denotes the relative entropy or the Kullback-Leibler divergence (cf.~\eqref{eq:def-rel-ent}).
See \eqref{eq:p1dual} for the exact dual problem.
Note the similarity between the above variational problem and~\eqref{eq:p1-1}: the two problems are equivalent when the regularizer $\epsilon$ is zero.
However, for non-zero~$\epsilon$, we enforce absolute continuity of the \textit{test measure} $q$ with the \textit{base measure} $p$ and exclude singular choices of $q$, such as the optimizer of \eqref{eq:p1-1}.
The dual formulation enables stability of training throughout while adaptively improving parts of the domains with higher residuals.
\subsection{A dual reformulation}
\label{sec:dual-formulation}

To arrive at the dual form of \eqref{eq:p1}, we use the Laplace principle (\ref{prop:laplace-principle}), which writes the maximum as a limit of integrating an increasing singular exponential functional
\begin{align*}
    \eqref{eq:p1} = \min_{\theta \in \mathcal T} \sup_{\epsilon > 0} \left\{ \epsilon \log \int_\Omega e^{r(x; \theta)/\epsilon} p(dx) \right\} \geq \sup_{\epsilon > 0} \min_{\theta \in \mathcal T} \left\{ \epsilon \log \int_\Omega e^{r(x; \theta)/\epsilon} p(dx) \right\}.
\end{align*}
To arrive at the desired form, we further express the inner functional using the Gibbs variational formula (\ref{prop:gibbs-variational})
\begin{align}
    \max_{\epsilon > 0} \min_{\theta \in \mathcal T} \max_{q \in \mathcal P(\Omega)} \left\{ \int_\Omega r(x;\theta) q(dx) - \epsilon \mathbf H(q | p) \right\}
    \label{eq:p1dual}
    \tag{P1*}
\end{align}
where $\mathbf H$ is the relative entropy
\begin{align}
    \mathbf H(q | p) = 
    \begin{dcases}
        \int_\Omega \log \frac{dq}{dp}(x) q(dx) &\text{if the Radon-Nikodym derivative}~\frac{dq}{dp}~\text{exists}, \\
        +\infty &\text{otherwise}.
    \end{dcases}
    \label{eq:def-rel-ent}
\end{align}
% Note that by picking $q = p$, we recover the usual $L^2$ loss, that is,
% \begin{align*}
%     \eqref{eq:p1} \geq \eqref{eq:p1dual} \geq \eqref{eq:p2}.
% \end{align*}
The innermost optimization admits a closed-form solution (see \ref{prop:gibbs-variational}) which reads 
\begin{align*}
    q^*(dx) = \frac {e^{r(x; \theta)}} {\int_\Omega e^{r(x';\theta)} p(dx')} p(dx) = \arg \max_{q \in \mathcal P(\Omega)} \left\{ \int_\Omega r(x;\theta) q(dx) - \epsilon \mathbf H(q | p) \right\}.
\end{align*}

\begin{remark}
The Laplace principle, referred to also as Varadhan's lemma \cite[Theorem 4.3.1]{dembo2009large}, has a long history in the area of statistical physics and large deviations theory.
In combination with a theorem from Bryc and the Gibbs variational principle, the Laplace principle gives an alternative characterization for large deviations principles and is foundational for the development of the \textit{weak convergence approach} to large deviations of Dupuis and Ellis \cite{dupuis2011weak, budhiraja2019analysis}; the computation used to derive the dual form is reminiscent of the techniques used in this weak convergence approach.
Similar exponential functionals are known in the community as the softmax function, commonly used in machine learning, Bayesian statistics, and transformers as a smooth approximation of the maximum function.
A similar family of probability distributions was recently used by Alberts and Bilionis \cite{alberts2023physics} for uncertainty quantification.
The inverse temperature $\beta$ plays the role of $1/\epsilon$ in this manuscript.
\end{remark}
The formulation of \eqref{eq:p1dual} gives rise to a natural alternating optimization scheme.
At each iteration, we solve the innermost optimization to the outermost.
Fix initial conditions $q^0, \theta^0, \epsilon^0$, define the iterations for the $k+1$-th iteration given the $k$-th:
\begin{align}
    \begin{dcases}
    &q^{k+1}(dx) \gets \frac{1}{\mathsf Z^k} \exp \left( \frac{r(x;\theta^k)}{\epsilon^k} \right) p(dx) ~~\text{where}~~ \mathsf Z^k = \int_\Omega \exp \left( \frac{r(x;\theta^k)}{\epsilon^k} \right) p(dx); \\
    &\theta^{k+1} \gets \mathsf{Minimize}\int_{\Omega} r(x,\theta^k) q^{k+1}(dx); \\
    &\epsilon^{k+1} \gets \mathsf{AnnealingScheme}(k, \epsilon^k, \theta^{k+1}, q^{k+1}).
    \end{dcases}
    \label{eq:alg-sketch}
\end{align}
There are two design choices.
The oracle $\mathsf{Minimize}$ refers to equation \eqref{eq:line-search} such as gradient descent and various higher-order methods.
On the other hand, $\mathsf{AnnealingScheme}$ refers to a schedule that incrementally decreases $\epsilon$.
For every fixed $\theta$ and $q$, the objective function in \eqref{eq:p1dual} is maximized by taking $\epsilon \to 0$ (\ref{prop:laplace-principle}).
However, since the optimization in $\theta$ takes place incrementally, the outer-loop optimization ought to take place in accordance with the inner loop.
For example, an annealing scheme used in numerical experiments is of the form
\begin{align*}
    \mathsf{AnnealingScheme}(k, \epsilon, \theta, q) = \frac{c \max_\Omega u(\cdot;\theta)}{\log 2 + k},
\end{align*}
which approaches zero as $k \to \infty$.

We contrast the proposed algorithm with simulated annealing \cite{vcerny1985thermodynamical, kirkpatrick1983optimization, geman1986diffusions}---which was the source of the terminological choice and the annealing schedule \cite[Theorem 1]{geman1986diffusions}---where given a sufficiently smooth potential $V: \Theta \to \R$, one finds the solution to $\min_\Theta V$ to be the long-time behavior ($t \to \infty$) of the solution to the stochastic differential equation
\begin{align*}
    d\theta_t = - \nabla V(\theta_t) dt + \sqrt{2\epsilon_t} dB_t
\end{align*}
where $B$ is a standard Brownian motion and $\epsilon_t$ is an annealing scheme (that is, $\epsilon^k \to 0$ as $k \to \infty$).
On the other hand, a formal functional central limit theorem suggests that the parameters obey the stochastic equation
\begin{align*}
    d\theta_t = - \mathbb E_{q_t} \left[ \nabla_\theta r(\cdot; \theta)\right] dt + \sqrt{\text{Var}_{q_t} \left[ \nabla_\theta r(\cdot; \theta)\right]} dB_t
\end{align*}
where $q_t$ depends on $\epsilon_t$ through
\begin{align*}
    q_t(dx) \propto e^{r(x; \theta_t)/\epsilon_t} p(dx).
\end{align*}
First, we remark that the ``potential'' in this case varies with $\epsilon_t$.
Also, unless $r(\cdot; \theta)$ has a unique maximum for all $\theta$, the limit of $q_t$ is non-degenerate and the variance does not vanish, which marks a second difference.
For these reasons, we simply draw formal connections between vRBA and simulated annealing but do not claim rigorous equivalences.

\subsection{Generalization of dual formulation}
\label{sec:general-duality}

Recall that Laplace's principle rewrites the $L^\infty$-norm by exponentially weighting spatial domains with large residuals which dominate the integral.
One can argue that similar concentration phenomena should hold for pairings besides the canonical $\log$-$\exp$ one.
In this subsection, we pursue this generalization to obtain a variational representation for RBA of the same form.
Our strategy is an inverse one: we start with a representation formula of form \eqref{eq:p1dual} and attempt to reverse-engineer the primal minimization objective.

Consider \textit{potential function} $\Phi: \R_+ \to \R_+$ to be convex, bounded from below, and superlinear, that is, 
\begin{align*}
    \lim_{r \to \infty} \frac{\Phi(r)}{r} = + \infty.
\end{align*}
The starting point of the generalization is to replace the relative entropy with the more general $\Phi$-divergence, defined by
\begin{align*}
    \mathbf D_\Phi(q|p) \coloneqq
    \begin{dcases}
        \int_\Omega \Phi \left( \frac{dq}{dp}(x) \right) p(dx) &\text{if the Radon-Nikodym derivative}~\frac{dq}{dp}~\text{exists}, \\
        +\infty &\text{otherwise.}
    \end{dcases}
\end{align*}
For strictly convex, superlinear $\Phi$, the statistical divergence is also convex, lower semicontinuous, and zero if and only if $q = p$.

We propose the generalized dual optimization problem:
\begin{align}
    \sup_{\epsilon > 0} \min_{\theta \in \mathcal T} \sup_{q \in \mathcal P(\Omega)} \left\{ \mathbb E_q \left[r(x;\theta) \right] - \epsilon \mathbf D_{\Phi^*}(q | p) \right\}
    \label{eq:p1-phi}
\end{align}
where $\Phi^*$ is the convex conjugate (or the Legendre-Fenchel transform) of $\Phi$, i.e.,
\begin{align*}
    \Phi^*(s) = \sup_{s \in \R} \left\{ sr - \Phi(r) \right\}.
\end{align*}
Restricting $\Phi$ to convex, suplinear functions guarantees that $\Phi^*$ is the same.
Similar to relative entropy, the innermost optimization also admits a generalized Gibbs variational represetation \cite{birrell2022f} (see \ref{prop:generalized-gibbs} for a statement), albeit with a more complicated form:
\begin{align*}
    \inf_{\nu \in \R} \left\{ \nu + \mathbb E_p[\Phi(r - \nu)] \right\} = \sup_{q \in \mathcal P(\Omega)} \left\{ \mathbb E_q \left[r(x;\theta) \right] - \epsilon \mathbf D_{\Phi^*}(q | p) \right\}
\end{align*}
From here and following steps in Section \ref{sec:dual-formulation}, it follows that \eqref{eq:p1-phi} is a dual formulation for minimizing the following norm-like quantity:
\begin{align}
    \min_{\theta \in \mathcal T} \sup_{\epsilon > 0} \left\{ \Lambda_\epsilon(r) \right\} ~~\text{where}~~\Lambda_\epsilon(r) \coloneqq \inf_{\nu \in \R} \left\{ \epsilon \Phi^{-1} \left( \frac{\nu}{\epsilon} + \int_\Omega \Phi \left( \frac{r(x;\theta) - \nu}{\epsilon} \right) p(dx) \right) \right\}
    \label{eq:p1-phi-lambda}
\end{align}
where we reparametered $\nu \mapsto \nu/\epsilon$ and moved the infimum to the outside as $\Phi^{-1}$ is monotonic.

In the absence of a specific $\Phi$, the functional $\lim_{\epsilon \to 0} \Lambda_\epsilon$ is difficult to interpret.
We enumerate several special cases that can be of interest.
\begin{enumerate}
    \item $\Phi(r) = e^r$ recovers the Laplace's principle.
    \item $\Phi(r) = r^2 + 1$, the corresponding $\Phi^*$-divergence is the chi-squared divergence and the objective function in the primal problem is the standard deviation (cf.~\ref{prop:quadratic-potential}): 
    \begin{align}
        \lim_{\epsilon \to 0} \Lambda_\epsilon(r) = \sqrt{\mathbb E_p[r^2] - \mathbb E_p[r]^2}.
        \label{eq:p1-phi-quad}
    \end{align}
    Note that this primal problem is not well-posed: in the context of PINNs, the solutions learned are the ones of the form $\mathcal F[u] = c$ for any $c \in \R$.
    However, by choosing the annealing scheme $\epsilon^k$ a particular way, one could obtain a dual formulation that is well-posed and corresponds to the $L^2(\Omega)$-norm (see discussion below).
    \item Taking $\nu = 0$, when the inequality
    \begin{align*}
        p \Phi(r) \leq r \Phi'(r) \leq q \Phi(r)
    \end{align*}
    is satisfied for some $p, q \in (1, \infty)$---which includes the case where $\Phi(r) = r^p _ c$---then by \cite[Equation 3]{fiorenza1991inequality}, we have
    \begin{align*}
        \frac{1}{c} \|r\|_{L^p(\Omega)} \leq \Lambda_\epsilon(r) \leq c \|r\|_{L^q(\Omega)} ~~\text{for}~c = \left( \frac{p}{q} \right)^{1/p}.
    \end{align*}
\end{enumerate}

A crucial step component is lost in this generalization: the generalized Gibbs variational formula lacks a general form of the optimizer in the innermost optimization of \eqref{eq:p1-phi}.
In particular, we only know that $q$ is optimal when
\begin{align*}
    q(dx) = \Phi' \left( \frac{r(x;\theta)}{\epsilon} \right) p(dx),
\end{align*}
which can only hold under an additional normalization condition
\begin{align}
    \int_\Omega \Phi' \left( \frac{r(x;\theta)}{\epsilon} \right) p(dx) = 1,
    \label{eq:p1-phi-normalization}
\end{align}
and is not always achieved.
Hence, we rely on the choice of the annealing scheme to ensure \eqref{eq:p1-phi-normalization} is always satisfied.
For example, if $\Phi(r) = r^2 + 1$, $\Phi'(r) = 2r$ and $\epsilon^k$ can be chosen directly to be whatever constant normalizes the kernel
\begin{align*}
    q^k(dx) = \frac{r(x;\theta^k)}{\int_\Omega r(x;\theta^k) p(dx)} p(dx).
\end{align*}
Satisfying \eqref{eq:p1-phi-normalization} implies that the optimal $\nu$ in \eqref{eq:p1-phi-lambda} is achieved at zero (Item 2 of \ref{prop:generalized-gibbs}), which recovers the canonical choice of minimizing the $L^2(\Omega)$-norm.

\begin{remark}
The functional $\Lambda_\epsilon$ has connections to the \textit{Jensen sums}, which are studied in convex analysis and have relations with \textit{Orlicz spaces} $L^\Phi$.
For convex, superlinear potentials~$\Phi$, Orlicz spaces are Banach spaces equipped with the Luxembourg norm
\begin{align*}
    \| f \|_\Phi = \inf \left\{ \epsilon > 0: \int_\Omega \Phi(|f(x)|/\epsilon) p(dx) \leq 1\right\}.
\end{align*}
Due to the similarity of forms, we conjecture connections between the asymptotics of the functional $\lim_{\epsilon \to 0} \Lambda_\epsilon$ to Orlicz spaces.
However, to limit the scope of the paper, we defer a detailed theoretical and numerical investigation of general choices of $\Phi$---as well as implications to optimization and interpretations in the context of PINNs and related machine learning tasks---to future work.
\end{remark}

\subsection{Discretization}
\label{sec:discretization}
An outstanding issue in implementing \eqref{eq:alg-sketch} is the tractability of integrating against the measure $q^k$.
Even when $p$ is a simple measure, e.g., uniform on a domain $\Omega$ or Gaussian, the exponential tilt---particularly the normalizing constant $\mathsf Z^k$---is typically expensive to compute as it requires fine discretization of the mesh.
Generally, updating $q^k$ only after several iterations suffices to reduce the computational complexity.

However, when $p$ is a distribution easy to sample from, one can approximate the expectation via Monte Carlo and a change-in-measure.
That is, let $(X_i)_{i=1}^n$ be independent and identically-distributed (i.i.d.)~samples from $p$.
Then, we can approximate $\mathbb E_{q^k}[r]$ for any measurable functional $r: \Omega \to R$, i.e., the residual or its gradient, by
\begin{align*}
    \mathbb E_{q^k}[r] = \mathbb E_p \left[ r \frac{dq^k}{dp} \right] = \lim_{n \to \infty} \frac{1}{n}\sum_{i=1}^n \frac{dq^k}{dp}(X_i) r(X_i).
\end{align*}
This is an \textit{unbiased} estimator for the expectation of $r$ over $q^k$.
However, in this context, the Radon-Nikodym derivative has an integral (the normalizing constant) that is difficult to compute. 
Therefore, we resort to approximating the integral with Monte Carlo samples, e.g., 
\begin{align*}
    \mathbb E_{q^k}[r] = \lim_{n \to \infty} \hat r^n \coloneqq \dfrac{ \sum_{i=1}^n e^{r(X_i;\theta^k)/\epsilon^k} r(X_i)}{\sum_{j=1}^n e^{r(X_j;\theta^k)/\epsilon^k}} \quad\quad p/q^k\text{-almost-surely}.
\end{align*}
In particular, we can define the \textit{multipliers}
\begin{align}
\label{exponential_q}
    \lambda_i^k \coloneqq \frac{e^{r(X_i; \theta^k)}}{\sum_j e^{r(X_j; \theta^k)}} ~~\text{which gives}~~ \hat r^n = \sum_{i=1}^n \lambda_i^k r(X_i),
\end{align}
and hence we recover the form of the estimator popularly used in the PINNs community.
Estimating the normalizing constant introduces bias to the estimator, though the estimator is nonetheless consistent and enjoys the central limit
\begin{align*}
    \lim_{n \to \infty} \sqrt{n} \left( \hat r^n - \mathbb E_{q^k}[r] \right) \rightarrow \mathsf N\left( 0, \text{Var}_{p} \left[ r \dfrac{dq^k}{dp} \right] \right) ~~\text{in distribution},
\end{align*}
where $\mathsf N(\mu, \sigma^2)$ is a Gaussian measure with mean $\mu$ and variance $\sigma^2$.
In other words, the discretization error vanishes on the order of $\mathcal O(1/\sqrt n)$ with constants that scale with the variance.
In the context of RBA, if the following inequality holds:
\begin{align*}
    \text{Var}_{p} \left[ r \dfrac{dq^k}{dp} \right] \leq \text{Var}_{p} \left[ r \right],
\end{align*}
then the discretization error obtained from reweighting is lower than that without; this indeed is observed empirically as demonstrated in the later sections.

As a final remark particularly relevant for operator learning: if the domain is a product space $\Omega = \Omega_1 \times \Omega_2$ equipped with a product measure $p(dx_1, dx_2) = p(dx_1) p(dx_2)$, the Radon-Nikodym derivatives can be disintegrated to facilitate implementation.
More specifically, one rewrites
\begin{align}
\label{Operator_loss_continous}
    \mathbb E_{q^k}[r] = \int_\Omega \frac{dq^k}{dp}(x) r(x) p(dx) = \int_{\Omega_1} \frac{dq^k}{dp}(x_2) \int_{\Omega_2} \frac{dq^k}{dp}(x_1|x_2) r(x_1, x_2) p(dx_1) p(dx_2)
\end{align}
where
\begin{align*}
    \frac{dq^k}{dp}(x_1|x_2) = \frac{\frac{dq^k}{dp}(x_1, x_2)}{\int_{\Omega_2} \frac{dq^k}{dp}(x_1, x_2) p(dx_1)} ~~\text{and}~~ \frac{dq^k}{dp}(x_2) = \frac{\int_{\Omega_2} \frac{dq^k}{dp}(x_1, x_2) p(dx_1)}{\int_\Omega \frac{dq^k}{dp}(x) p(dx)}.
\end{align*}
From here, one can choose to sample or weight either coordinate.
For example, in operator learning, the domain is a product of some function space and the spatial domain that the functions is defined on.
We then compute the vRBA estimate by sampling over function space and computing multipliers for the spatial domain.

\section{Methods}
\label{sec:methods}

\subsection{Variational Residual-based Attention Methods}
\label{vrba_methods}

As described in the previous section, the general method is flexible, allowing for the use of different potential functions, $\Phi$, with particular properties.

The training process starts by sampling $N$ i.i.d.~random variables $\{X_i\}_{i=1}^N$ uniformly from the domain $\Omega$ and calculating the corresponding residuals $r(X_i)$. The general method then involves three steps per iteration: (1) updating the tilted distribution $q$, (2) updating the model parameters $\theta$ via a line search method, and (3) updating the temperature parameter $\epsilon$ using an annealing scheme.

\subsubsection{Update the tilted distribution}

While the generalized Gibbs variational principle does not typically yield a closed-form update rule for the distribution $q$, one can obtain the optimal tilt by choosing specific potentials $\Phi$ and annealing schedules (to be discussed in the coming subsection).
Under the appropriate choices, $q$ is proportional to the derivative of the potential $\Phi'$.
For a discrete set of points, the optimal probability mass function (p.m.f.) for the next iteration, $q^{k+1}$, is given by
\begin{align}
\label{general_q}
    q^{k+1}(X_i) \propto \Phi' \left( \frac{r(x_i;\theta^k)}{\epsilon^k} \right)
\end{align}
We consider two particular cases for $\Phi$.

\paragraph{Case I: $\Phi(x)=e^x$} This choice of potential recovers the Laplace principle, which is associated with minimizing the $L^{\infty}$ norm, and there is an exact convex duality.
The derivative is $\Phi'(x)=e^x$, so the optimal distribution is proportional to $\exp(|r|/\epsilon)$ under no additional assumptions. 
This resulting distribution is given by a variation of the \textit{softmax} of the scaled residuals akin to attention mechanisms popularized in the transformer architecture \cite{vaswani2017attention}
\begin{align}
\label{exp_dist}
    q^{k+1}(X_i,\theta^k) =\frac{\exp \left( \frac{r(X_i;\theta^k)}{\epsilon^k} \right) }{\sum_{j=1}^N \exp \left( \frac{r(X_j;\theta^k)}{\epsilon^k} \right) } 
\end{align}

\paragraph{Case II: $\Phi(x)= x^2 + 1$} In this case, the derivative is $\Phi'(x)= 2x$, which is associated with minimizing the $L^2$ norm. 
In this case, the convex duality is not exact, and the appropriate annealing schedule will be chosen shortly.
The discretized p.m.f. takes the following form
\begin{align}
\label{quad_dist}
    q^{k+1}(X_i,\theta^k) =\frac{|r(X_i,\theta^k)|}{\sum_{j=1}^N |r(X_j,\theta^k)|}
\end{align}

Finally, since the collocation points $\{X_i\}$ are randomly sampled and the target distribution $q^{k+1}$ depends on the current model parameters, the resulting p.m.f. can be prone to spurious fluctuations. To promote stability, especially when using fast-growing potentials like $\Phi(r)=e^r$ that can create sharply peaked distributions, we smooth the target distribution over time using an exponential moving average (EMA).

Furthermore, we introduce an additional smoothing mechanism by interpolating between the adaptive distribution $q^{k+1}$ and the base uniform distribution $p_u$. The combined update rule for the importance weights reads
$$
\lambda_i^{k+1} = \gamma \lambda_i^k + \eta^* \left( \phi q^{k+1}(X_i, \theta^k) + (1-\phi) p_u(X_i) \right)
$$
where $\gamma \in [0,1)$ is a memory term and $\eta^*$ is a learning rate. The parameter $\phi \in [0,1]$ controls the degree of adaptivity; $\phi=1$ corresponds to the fully adaptive case from which the original method is recovered, while smaller values of $\phi$ increase stability by biasing the distribution towards uniformity. For stability reasons, which becomes particularly important for second-order methods, we have found that normalizing the learning rate as $\eta^*=\eta/\max_{\Omega} q^{k+1}$ is beneficial.

The resulting vector $\lambda^{k+1}$ can be interpreted as the smoothed, unnormalized distribution that guides the optimization. While it incorporates information from the optimal distribution $q^{k+1}$, it is not itself a probability mass function (p.m.f.) as it does not necessarily sum to one.

\subsubsection{Update the model parameters}

Once the smoothed importance weights $\{\lambda_i^{k+1}\}$ have been computed, they can be used to formulate the loss function in two primary ways: by guiding a resampling process or by directly weighting the residuals.

\paragraph{Importance Sampling}
In this approach, the weights are first normalized to recover a smooth probability mass function (p.m.f.) over the discrete domain $\Omega$ 
\begin{align*}
    \bar{q}_i^{k+1} = \frac{\lambda_i^{k+1}}{\sum_j \lambda_j^{k+1}}
\end{align*}
This new distribution $\bar{q}$ is then used to {resample} a new set of training points $\{X_i\}_{i=1}^N$ from the full collocation set $\Omega$. By focusing the sampling on high-importance regions, the loss can be computed as a standard, unweighted mean squared error on this new, more challenging set of points
\begin{align}
\label{RAR_loss}
    \mathcal{L}(\theta^{k+1}) = \frac{1}{N}\sum_{i=1}^N r(X_i, \theta^{k+1})^2.
\end{align}
Notably, this framework can recover methods similar to residual-based adaptive sampling \cite{wu2023comprehensive} by setting the potential to $\Phi(x) = x^2 + 1$ and the EMA parameters to $\eta^*=1$ and $\gamma=0$.

\paragraph{Importance Weighting}
Alternatively, the weights can be used directly in an importance weighting scheme. In this case, the training points $\{X_i\}_{i=1}^N$ are sampled uniformly from $\Omega$. The weights $\{\lambda_i^{k+1}\}$ are then applied directly to the residuals within the loss calculation, creating a weighted objective that reads
\begin{align}
\label{RBA_loss}
    \mathcal{L}(\theta^{k+1}) = \frac{1}{N}\sum_{i=1}^N [\lambda_i^{k+1}r( X_i, \theta^{k+1})]^2.
\end{align}
One might worry that squaring the weights and residual depart from the procedure outlined previously.
By an application of Jensen's inequality, one can obtain that squaring the residual (and weights when applicable) corresponds to solving a strictly stronger problem with the added benefits of differentiability.

This framework is also general enough to recover other popular methods. For example, with the potential $\Phi(x) = x^2+1$ and specific choices of EMA parameters, it is possible to recover the traditional residual-based adaptive weights~\cite{anagnostopoulos2024residual} and their variations~\cite{toscano2025aivt, toscano2025kkans}. Similarly, the formulation proposed in~\cite{si2025convolution} can be recovered by constructing the distribution using a locally averaged residual. Our perspective can also be used to interpret other advanced heuristics; for instance, methods that balance the residual decay rate~\cite{chen2025self} can be viewed as replacing the simple temporal smoothing via EMA with a more sophisticated, history-aware mechanism for computing the adaptive distribution $q^k$.

Once the loss $\mathcal{L}(\theta^{k+1})$ is computed using either method, the model parameters can be updated via a line search algorithm, as Adam \cite{kingma2014adam} or BFGS variations such as SSBroyden \cite{urban2024unveiling}.

\subsubsection{Update the regularization parameter}

As alluded to before, the choice of the annealing schedule depends on the choice of potential $\Phi$.

\paragraph{Case I: Exponential Potential ($\Phi(x)=e^x$)}
For this choice, there is no requirements on the choice of $\epsilon$. 
The particular choice we implemented reads
\begin{align*}
    \epsilon^{k+1}=\frac{c\max_{\Omega} r(x;\theta^k)}{\log(2+k)}.
\end{align*}
This schedule has several advantages. 
Using the maximum residual, $\max_{\Omega}|r|$, in the numerator provides a dynamic, problem-dependent {characteristic scale} for the temperature $\epsilon$. 
This helps to normalize the magnitude of the residuals relative to the magnitude of the solution itself. 
The logarithmic term in the denominator ensures a slow and stable decay, such that $\epsilon^k \to 0$ as $k \to \infty$, which gradually sharpens the distribution's focus on the largest residuals.
In the context of simulated annealing, logarithmic decay is sufficiently slow to guarantee global convergence \cite[Theorem 1]{geman1986diffusions}.

\paragraph{Case II: ($\Phi(x) = x^2 + 1$)}

For this case, the optimality of the distribution $q$ holds under the normalization condition \eqref{eq:p1-phi-normalization}, that is,
\begin{align*}
    \int_\Omega \Phi' \left( \frac{r(x;\theta)}{\epsilon} \right) p(dx) = 1.
\end{align*}
For this case, the constraint is satisfied by choosing $\epsilon^k$ to be the normalizing constant, i.e., the optimal tilt is
\begin{align*}
    q^k(dx) = \frac{r(x;\theta^k)}{\int_\Omega r(x';\theta^k) \, p(dx')} p(dx).
\end{align*}
This is a consequence of the simple structure of the quadratic potential.
Such simplications can hold for more general polynomial potentials (with additive constants) as well.
% In this expression, the parameter $\epsilon^k$ appears in both the numerator and the denominator and therefore {cancels out entirely}:
% \begin{align*}
%     q^k(dx) = \frac{|r(x;\theta^k)|}{\int_\Omega |r(x';\theta^k)| p(dx')} p(dx)
% \end{align*}
% This means that for the quadratic potential, the optimal distribution $q$ is actually {independent of $\epsilon$}. The normalization is handled directly by the integral of the absolute residuals, and no annealing schedule for $\epsilon$ is required.

\subsection{Physics-Informed Neural Networks (PINNs)}
\smallskip

In the PINN framework, the goal is to approximate the solution $\bar{u}(x)$ of a PDE or ODE using a representation model $u(\theta, x)$. The training objective ensures that the model satisfies the governing equations and any available data, which may include boundary conditions, initial conditions, or sparse observations.

The total loss function combines the individual loss terms for the governing equations ($\mathcal{L}_E$), boundary, or initial conditions ($\mathcal{L}_B$), and, for inverse problems, observational data ($\mathcal{L}_D$), which reads
\begin{align}
\label{PIML_loss}
    \mathcal{L} = m_E\mathcal{L}_E + m_B\mathcal{L}_B+m_D\mathcal{L}_D,
\end{align}
where $m_E, m_B, \text{and } m_D$ are global weights that balance the contribution of each term. The individual loss functions ($\mathcal{L}_E$, $\mathcal{L}_B$, $\mathcal{L}_D$) are each computed as described in equation~\eqref{RAR_loss} for the importance sampling approach or as in equation~\eqref{RBA_loss} for the importance weighting. To ensure the update directions induced by the different loss components are balanced, we employ the self-scaling mechanism presented in~\cite{toscano2025kkans}. A detailed description of the proposed method is presented in algorithm~\ref{vRBA_alg}.
\subsection{Operator Learning}

Neural Operators (NOs) learn mappings between function spaces, approximating a solution operator $G_\theta$ that takes an input function, such as a source term $f$, and maps it to the corresponding solution $u$ \cite{lu2021learning, li2020fourier}. They can also be formulated to propagate a solution in time, taking $u(t_0)$ and returning $u(t_0+\Delta t)$. While several variations of operators exist, this study focuses on three popular architectures: DeepONet \cite{lu2019deeponet}, Fourier Neural Operators (FNOs) \cite{li2020fourier}, and the Time-conditioned U-Net \cite{ovadia2025real, gupta2022towards}. A detailed description of these is provided in \ref{Operators_descriptions}. The framework described in Section~\ref{vrba_methods} can be extended to this case, which requires updating our notation.

Let $\mathcal{X}$ be a space of functions over a domain $\Omega_X \subset \mathbb{R}^{d_x}$, and $\mathcal{Y}$ be a space of functions over $\Omega_Y \subset \mathbb{R}^{d_y}$. The operator of interest is
\[
\mathcal{G}: \mathcal{X} \ni v \mapsto \mathcal{G}[v] \in \mathcal{Y}.
\]
The goal is to learn a parametric model $G_\theta$ that approximates $\mathcal{G}$. The residual $R:\mathcal{X}\times \Omega_y\to\mathbb R^+$ for this task is defined as the difference between the operator's prediction and the true solution and reads
\begin{align}
    \label{residual_operator}
    R(v, x; \theta) =| G_\theta(v)(x) - \mathcal G[v](x)|, 
\end{align}
where $v \in \mathcal{X}$ is an input function and $\mathcal{G}[v]$ is the corresponding true output function evaluated at a point $x \in \Omega_Y$. The training data consists of $N_{\text{func}}$ input-output function pairs, $\{v_j, \mathcal G[v_j]\}_{j=1}^{N_{\text{func}}}$, where each output function $\mathcal G[v_j]$ is evaluated at $N$ discrete points $\{x_i\}_{i=1}^{N}$. The standard loss is an average over both the function instances and the spatial points:
\begin{equation}
\label{NO_loss}
    \mathcal{L}(\theta) = \frac{1}{N_{\text{func}}}\sum_{j=1}^{N_{\text{func}}} \frac{1}{N} \sum_{i=1}^{N}[R(v_j, x_i; \theta)]^2
\end{equation}

A single importance sampling or weighting scheme is ill-suited for this problem due to the two distinct levels of discretization (in function space and spatial domains). To address this, we propose a mixed strategy: {importance weighting} is used for the spatial points within each function, while {importance sampling} is used for the functions themselves. This is motivated by the fact that many NOs have a fixed spatial discretization, making weighting a natural fit, while the function space offers more flexibility for sampling.

The loss function for a batch of $b_u$ functions is updated as follows
\begin{equation}
\label{NO_loss_RBA}
    \mathcal{L}(\theta) = \frac{1}{b_u}\sum_{j=1}^{b_u} \frac{1}{N} \sum_{i=1}^{N}[\Lambda_{i,j} R(v_j, x_i; \theta)]^2,
\end{equation}
where the functions $\{v_j\}_{j=1}^{b_u}$ are sampled from the full set of training functions. The term $\Lambda$ is a matrix of importance weights, where $\Lambda_{i,j}$ corresponds to point $x_i$ for function $v_j$. These weights are constructed from a target p.m.f. matrix $Q^k$ constructed based on the choice of potential. For instance when $\Phi(x)=e^{x}$, $Q^k \in \R^{N \times N_{\text{func}}}$ is defined recursively as follows
\begin{align*}
    Q_{i,j}^{k+1}(\theta^k)=\frac{\exp\left(\frac{R(v_j,x_i; \theta^k)}{\epsilon^k}\right)}{\sum_{\ell=1}^{N} \exp\left(\frac{R(v_j,x_\ell; \theta^k)}{\epsilon^k}\right)}. 
\end{align*}
Note that each column of the matrix $Q$ (for a fixed function $j$) is a p.m.f. over the spatial points, focusing attention on high-residual regions for that specific function. The weights are then smoothed over time with an EMA
\begin{align}
\label{lambdas_update_NOs}
    \Lambda_{i,j}^{k+1}=\gamma\Lambda_{i,j}^{k}+\eta^*Q_{i,j}^{k+1}(\theta^k). 
\end{align}
As in the previous case, we can set the learning rate for stability, for example, by normalizing it as $\eta^*=\eta/\max_{j} Q_{i,j}$. Note that this choice of $\eta^*$ achieves a normalization per function which is consistent with our two-level discretization. This EMA formulation has the useful property of keeping the weights bounded. As described in~\cite{anagnostopoulos2024residual}, the update rule ensures that the weights are constrained to the interval $\Lambda_{i,j}\in(0,\frac{\eta^*}{1-\gamma})$, which aids in stabilizing the training process.

A key advantage of this framework is that, if $\eta\neq 1-\gamma$, we can leverage the imbalance on learned spatial weights, $\Lambda_{i,j}$, to construct a sampling distribution over the \textit{functions} themselves. The intuition is that functions with higher overall residuals will naturally accumulate larger $\Lambda$ values over time. Therefore, we propose the following approach to create a function-level sampling distribution. First, we compute an aggregated importance score $s_j$ for each function by summing its spatial weights
\begin{align*}
    s_{j}=\sum_{i=1}^{N}\Lambda_{i,j}.
\end{align*}
These scores are then normalized to create a p.m.f. over the function space:
\begin{align*}
    \bar q_j=\frac{s_j}{\sum_{\ell=1}^{N_{func}}s_\ell}.
\end{align*}
This distribution $\bar{q}$ can then be used to sample the most informative functions $v_j$ for the next training batch. A detailed description of the proposed method is given in Algorithm~\ref {vRBA_NO_alg}
\section{Results}
\label{Results}

\subsection{Physics-Informed Neural Networks}

\subsubsection{Allen-Cahn Equation}

\begin{figure}[H]
    \centering
    \includegraphics[width=1\linewidth]{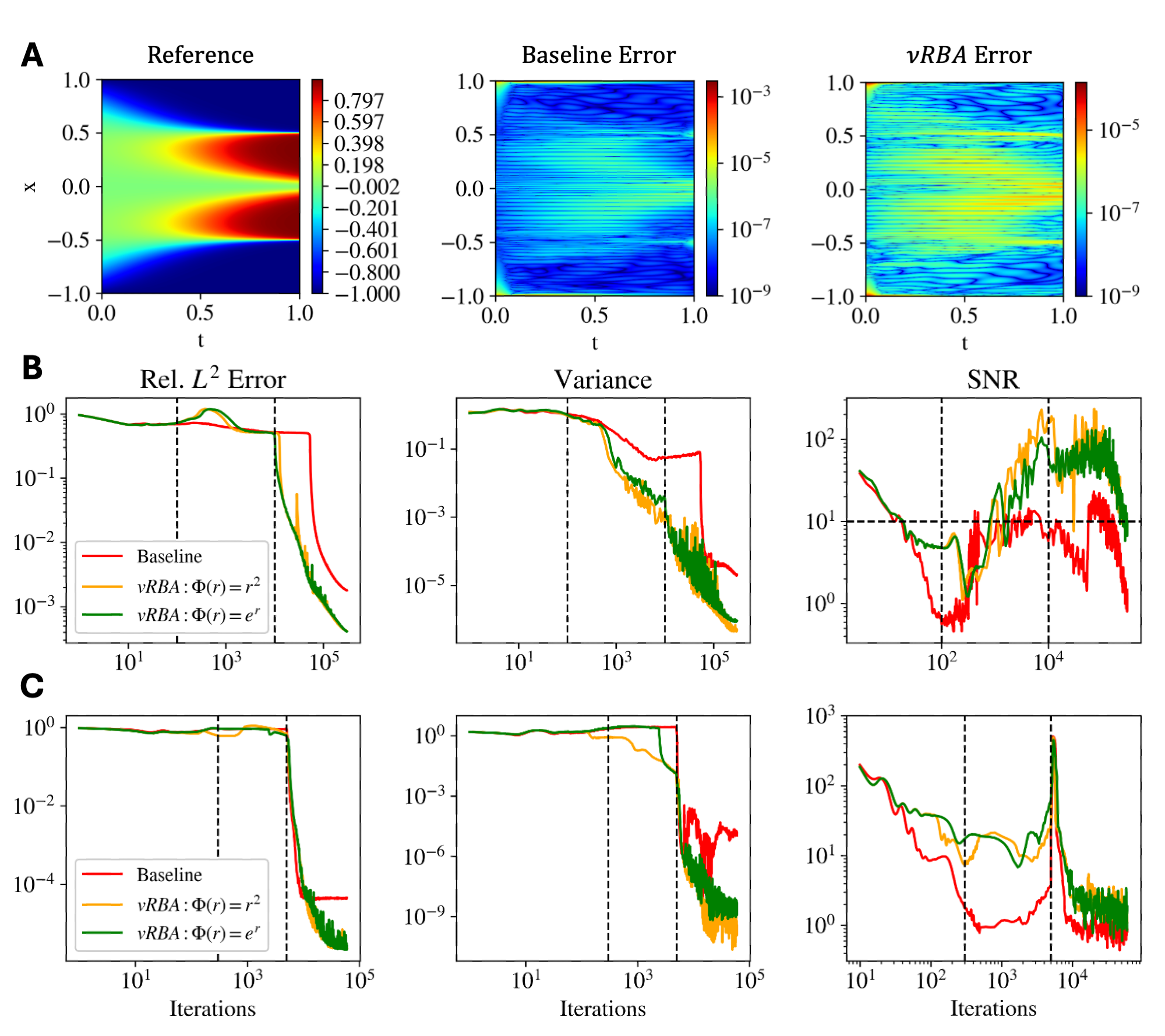}
\caption{
\textbf{Allen-Cahn Equation Results.}
\textbf{(A)} Comparison of the reference solution, baseline model error, and vRBA ($\Phi(r)=e^r$) error. The baseline's error is highly concentrated, while vRBA promotes a more uniform distribution, halving the maximum absolute error.
\textbf{(B)} Training dynamics with a first-order Adam optimizer, showing relative $L^2$ error, residual variance, and Signal-to-Noise Ratio (SNR). vRBA converges significantly faster ($\sim$5k iterations) than the baseline ($\sim$50k), substantially reduces residual variance (implying smaller discretization error), and maintains a higher SNR. All models show three learning stages, but vRBA transitions to the productive diffusion phase more rapidly, improving learning dynamics.
\textbf{(C)} Training dynamics with a second-order SSBroyden optimizer. Even with an advanced optimizer, vRBA achieves superior performance and faster convergence while the baseline stagnates in a local minimum. vRBA again reduces variance by over three orders of magnitude and maintains a much higher SNR, demonstrating that the adaptive scheme's benefits are independent of the optimizer.
} \label{AC}
\end{figure}

The Allen-Cahn equation is a widely recognized benchmark in PINNs due to its challenging characteristics. The 1D Allen-Cahn PDE is defined as:

\begin{equation} \frac{\partial u}{\partial t} = k\frac{\partial^2 u}{\partial x^2} - 5u(u^{2}-1), \end{equation}

\noindent where $k = 10^{-4}$. The problem is further defined by the following initial and periodic boundary conditions:

\begin{equation} u(0, x) = x^2 \cos(\pi x), \quad \forall x \in [-1, 1], \end{equation}

\begin{equation} u(t, x + 1) = u(t, x - 1), \quad \forall t \geq 0 \quad \text{and} \quad x \in [-1, 1]. \end{equation}

\begin{table}[t]
    \centering
    \begin{tabular}{clcccc}
        \toprule
        &\textbf{Model} & \textbf{N. Params} & \textbf{Optimizer}& \textbf{Time (ms/it)} & \textbf{Rel. $L^2$ Error} \\
        \midrule
        A&Baseline       & 21318 &Adam& 4.01 & $1.77 \times 10^{-3}$ \\
        &$vRBA:\Phi(r)=r^2$      & 21318&Adam   & 4.03  & $\mathbf{4.08 \times 10^{-4}}$ \\
        &$vRBA:\Phi(r)=e^r$       & 21318&Adam & 4.47 & $4.14 \times 10^{-4}$ \\
        \midrule
        B&Baseline& 2011&SSBroyden & 22.4 & $4.34 \times 10^{-5}$ \\
        &$v RBA:\Phi(r)=r^2$ &2011&SSBroyden & 24.3&$2.42 \times 10^{-6}$ \\
        &$v RBA:\Phi(r)=e^r$ & 2011&SSBroyden  & 23.7 & $\mathbf{2.14 \times 10^{-6}}$ \\
        \bottomrule
    \end{tabular}
\caption{Comparison of models for the Allen-Cahn equation, where baseline models using uniform sampling are benchmarked against the proposed vRBA method. The table presents two distinct scenarios: (a) vRBA applied as an importance weighting strategy for a large network using a first-order optimizer (Adam), and (b) vRBA applied as importance sampling for a compact network using a second-order optimizer (SSBroyden). Notably, using an adaptive method with either an exponential ($\Phi(r)=e^r$) or quadratic ($\Phi(r)=r^2$) potential reduces the relative $L^2$ error by approximately an order of magnitude, an improvement that holds true even when using a highly optimized second-order method.}

    \label{AC_PIML}
\end{table}
To provide a broad overview of the influence of vRBA, we split this example into two parts. In the first part, following prior work~\cite{zhang2023dasa,anagnostopoulos2024residual,wang2024piratenets, chen2024self}, we use a larger network with Fourier feature embeddings and a first-order Adam optimizer applying vRBA as an importance weighting mechanism. The quantitative results are shown in Table~\ref{AC_PIML}(A). The vRBA methods achieve a significantly lower final relative $L^2$ error, reducing it from $1.77 \times 10^{-3}$ in the baseline to as low as $4.08 \times 10^{-4}$. This more than four-fold improvement is achieved with a negligible increase in computational cost per iteration.

Similarly, Figure~\ref{AC}(B) shows that vRBA significantly accelerates convergence; the vRBA models begin their main convergence phase around 5,000 iterations, whereas the baseline does not start to converge until nearly 50,000 iterations. This enhanced performance can be explained by a twofold mechanism. First, vRBA reduces the variance of the loss estimator, which directly lowers the discretization error. Second, it induces a higher Signal-to-Noise Ratio (SNR), which improves the learning dynamics by enabling a much faster transition to the total diffusion phase. Together, these effects lead to faster convergence and superior model performance.

In the second part, we analyze our model's performance using the recently introduced SSBroyden optimizer~\cite{urban2024unveiling}, which has been successfully applied in PINNs to obtain highly accurate results~\cite{urban2024unveiling, kiyani2025optimizing,wang2024piratenets}. Following~\cite{urban2024unveiling}, we use Fourier embeddings to encode the periodic boundary conditions. The model parameters are initialized with 5,000 Adam iterations, after which we switch to the SSBroyden optimizer. During the SSBroyden phase, we periodically resample the collocation points every 100 iterations using vRBA as an importance sampling strategy. This approach notably enables the use of mini-batches for training, even with a second-order optimizer.

The results, summarized in Table~\ref{AC_PIML}(B), show that this strategy yields a substantial performance gain. The vRBA framework reduces the final relative $L^2$ error by over an order of magnitude, from $4.34 \times 10^{-5}$ for the baseline to $2.14 \times 10^{-6}$. Figure~\ref{AC}(A) provides a qualitative comparison of the final pointwise error for the SSBroyden experiment, showing the results for the vRBA model with an exponential potential. The baseline model's error is {highly concentrated} in specific horizontal bands across the domain. In contrast, the vRBA framework produces a much more {spatially uniform} error distribution. A significant quantitative gain matches this qualitative improvement: the maximum absolute error is reduced by nearly {two orders of magnitude}, from approximately $10^{-3}$ for the baseline to $10^{-5}$ for the vRBA model.

The learning dynamics in Figure~\ref{AC}(C) offer a more nuanced perspective on how this is achieved. During the initial 5,000 Adam iterations, the model is trapped in the transition stage, and the error for the vRBA models does not decrease, despite a relatively high SNR. The transition to the diffusion phase, where the error rapidly converges, happens immediately upon switching to the SSBroyden optimizer. This transition is signaled by a sharp, initial spike in the SNR, after which the SNR drops to a level lower than it was during the Adam phase. This behavior suggests that the critical event for convergence is the initial escape from the transition phase, rather than simply maintaining a high absolute SNR throughout the entire training process.
\begin{table}[H]
    \centering
    % The only change is in the line below, setting fixed widths
    \begin{tabular}{p{2.5cm} p{2cm} p{2cm} p{4cm} p{3cm}}
        \toprule
        \textbf{Problem} & \textbf{Reference} & \textbf{Optimizer}& \textbf{Enhancements} & \textbf{Rel. $L^2$ Error}\\
        \midrule
        Allen Cahn& \cite{wang2025gradient}&SOAP & PN, FF, WF,CS,
        LRA& $3.48 \times 10^{-6}$ \\
        
        &\cite{urban2024unveiling} &   SSBroyden& RAD, FF&
        $2.20\times10^{-6}$\\
        &TW& SSBroyden& $vRBA(\Phi=e^r)$, FF&$\mathbf{2.14\times10^{-6}}$\\
        \midrule
        Burgers& \cite{wang2025gradient}&SOAP & PN, FF, WF,CS,
        LRA& $4.03 \times 10^{-5}$ 
        \\
        &\cite{urban2024unveiling} &   SSBroyden& RAD, FF&
        $2.90\times10^{-8}$
        \\
        &\cite{kiyani2025optimizing} &   SSBroyden& RAD, FF, WLS&
        $1.62\times10^{-8}$\\
        
        &TW & SSBroyden&  $vRBA(\Phi=e^r)$,  FF&
        $\mathbf{1.51\times10^{-8}}$\\        \bottomrule
    \end{tabular}
\caption{State-of-the-art comparison of relative $L^2$ errors for the Allen-Cahn and Burgers equations, focusing specifically on the performance of various quasi-Newton (second-order) optimization methods. The results from this work (TW) are benchmarked against reported results in recent literature. As described in the preceding sections, the Residual-based Adaptive Distribution (RAD) is a specific type of vRBA that utilizes a quadratic potential without smoothing. The results in this table, therefore, underscore that the combination of an adaptive sampling strategy with a suitable optimizer is critical for achieving high accuracy. Enhancements from the literature are abbreviated as follows: PirateNet (PN)~\cite{wang2024piratenets}, Fourier Features embeddings (FF)~\cite{Wang2020_Fourier_nets}, Weight Factorization (WF)~\cite{wang2022random}, Causality (CS)~\cite{wang2022respecting}, Learning Rate Annealing (LRA)~\cite{wang2021understanding} and Wolfe Line Search (WLS).}
\label{Burgers_SOTA}
\end{table}

To place our results in context, we compare them against other state-of-the-art  methods for the Allen-Cahn equation in Table~\ref{Burgers_SOTA}. The comparison focuses on highly accurate solutions obtained with quasi-Newton optimizers. Our proposed method, combining the SSBroyden optimizer with vRBA, yields the most accurate result among the compared methods, achieving a relative $L^2$ error of $2.14 \times 10^{-6}$. It is essential to note that this performance is achieved with a minimal set of enhancements, only Fourier Features and our adaptive sampling framework. This contrasts with other approaches that rely on a more extensive suite of techniques yet yield less accurate results. Furthermore, the table highlights that the previous best result for this problem also uses an adaptive method, RAD. As detailed in our theoretical framework, RAD is a specific instance of vRBA that uses a quadratic potential without smoothing. This underscores our main conclusion: the combination of a powerful optimizer with a principled adaptive sampling strategy like vRBA is the critical factor for achieving the highest accuracy.

\subsubsection{Burgers Equation}
\begin{figure}[H]
    \centering
    \includegraphics[width=1\linewidth]{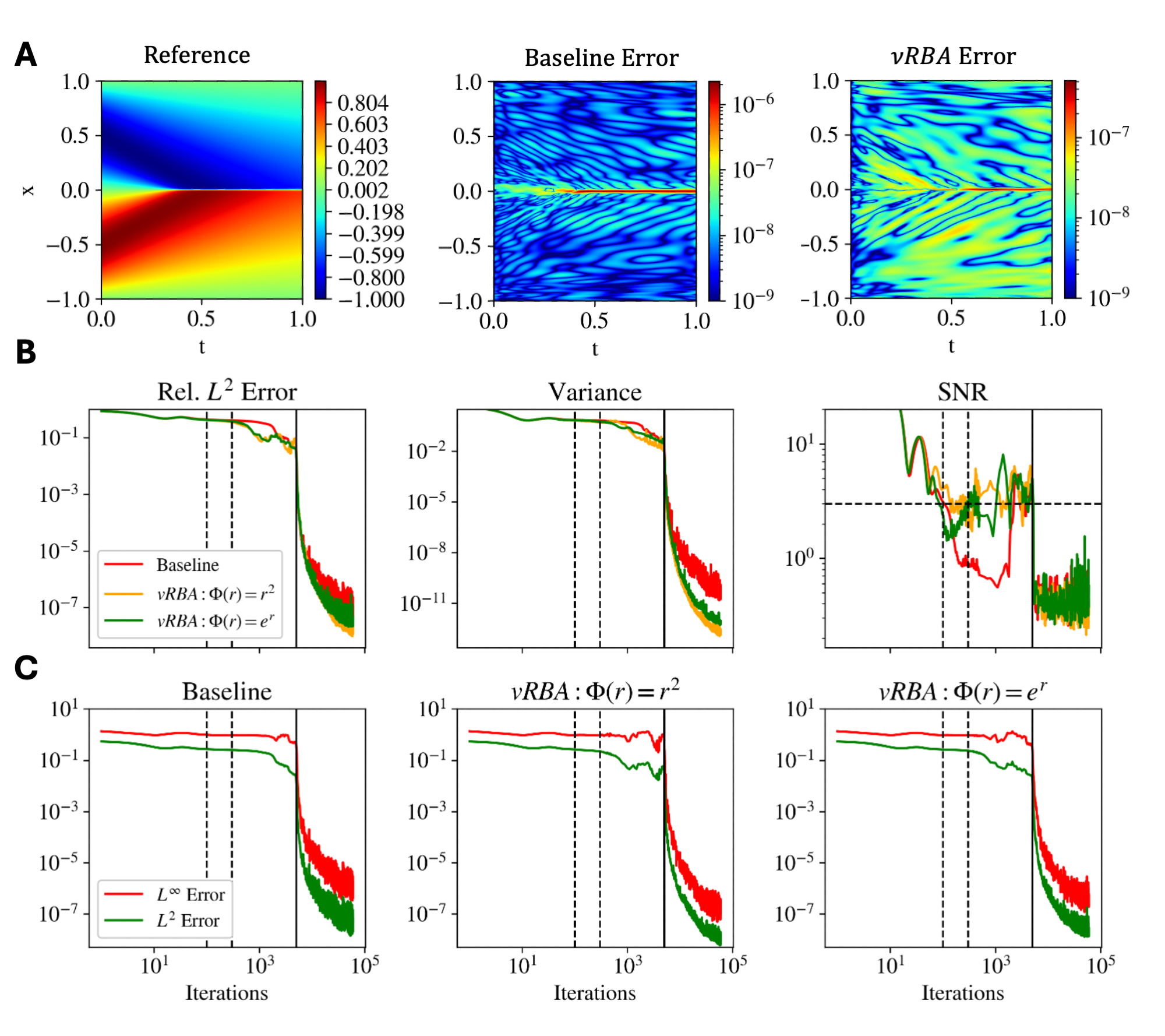}
\caption{
\textbf{Burgers' Equation Results.}
\textbf{(A)} Comparison of the model prediction, reference solution, and pointwise absolute error. The baseline model's error is highly concentrated along the pseudo-discontinuity (shock front), whereas the vRBA framework successfully distributes the error more uniformly across the spatiotemporal domain, resulting in a significantly lower overall error.
\textbf{(B)} Training dynamics showing the convergence of relative $L^2$ error, residual variance, and SNR with the SSBroyden optimizer. Both vRBA models achieve a lower final relative error and reduce the residual variance by three orders of magnitude. The learning dynamics are notably improved, as evidenced by the higher SNR, with the most pronounced gains occurring within the first 1,000 iterations.
\textbf{(C)} Evolution of the $L^2$ and $L^\infty$ error norms during training. For all methods, the two error norms exhibit similar convergence patterns. Crucially, the $L^\infty$ error consistently provides an upper bound for the $L^2$ error, visually confirming that its minimization is a stronger convergence criterion.
}
    \label{Burgers}
\end{figure}

The Burgers' equation is defined as
\begin{equation}
  u_t + uu_x = \nu u_{xx},
  \label{eq:Burgers}
\end{equation}
\noindent where $u$ represents the velocity field, subject to the dynamic viscosity $\nu=1/(100\pi)$. The initial condition and boundary conditions are described as follows
\begin{align}
  u(0, x) = -\sin(\pi x), \quad \forall x \in \Omega,\\
  u(t, -1) = u(t, 1) = 0, \quad \forall t \geq 0,
  \label{Burgers_BC}  
\end{align}

\noindent defined over the domain $\Omega = (-1,1) \times (0,1)$, where $\bm{x} = (x, y)$ signifies the spatial coordinates.

\begin{table}[H]
    \centering
    \begin{tabular}{lcccc}
        \toprule
        \textbf{Model} & \textbf{N. Params} & \textbf{Optimizer}& \textbf{Time (ms/it)} & \textbf{Rel. $L^2$ Error} \\
        \midrule
        Baseline& 2011&SSBroyden & 26.1 & $1.67 \times 10^{-7}$ \\
        $v RBA:\Phi(r)=r^2$ &2011&SSBroyden & 28.2&$1.74 \times 10^{-8}$ \\
        $v RBA:\Phi(r)=e^r$ & 2011&SSBroyden  & 27.6 & $\mathbf{1.51 \times 10^{-8}}$ \\
        \bottomrule
    \end{tabular}
\caption{Performance of the vRBA method for the Burgers equation. The vRBA strategies, using either a quadratic ($\Phi(r)=r^2$) or exponential ($\Phi(r)=e^r$) potential, are benchmarked against a baseline model with uniform sampling. All models are trained with the second-order SSBroyden optimizer. The results demonstrate that the vRBA methods improve the final relative $L^2$ error by an order of magnitude over the baseline, underscoring that an advanced adaptive sampling strategy is critical for achieving high accuracy.}
    \label{Burgers_PIML}
\end{table}

For this example, we follow previous work~\cite{urban2024unveiling}, and enforce the initial and boundary conditions using hard constraints and Fourier embeddings. In this setup, vRBA is again applied as an importance sampling strategy, with the collocation points being resampled every 100 iterations.

The results, presented in Table~\ref{Burgers_PIML}, demonstrate a clear performance advantage for our method, reducing the final relative $L^2$ error by over an order of magnitude, from $1.67 \times 10^{-7}$ to $1.51 \times 10^{-8}$.

A qualitative view of the final error is provided in Figure~\ref{Burgers}(A). The baseline model's error is highly concentrated along the moving shock front, while the vRBA framework produces a much more uniform error distribution. The learning dynamics in Figure~\ref{Burgers}(B) explain how this is achieved, following the same twofold mechanism observed previously. First, the vRBA models reduce the variance of the loss estimator by three orders of magnitude, which lowers the discretization error. Second, they induce a higher SNR and enable a significantly faster transition to the diffusion phase, resulting in improved learning dynamics and faster convergence. Figure~\ref{Burgers}(C) further confirms that the $L^\infty$ error norm consistently bounds the $L^2$ error norm.

Finally, we compare our results to the state of the art in Table~\ref{Burgers_SOTA}. Our approach, which combines the SSBroyden optimizer with the vRBA sampling framework, yields the most accurate result among the compared methods. This highlights that a principled adaptive strategy is a key component for pushing the boundaries of accuracy in PINNs.

\subsection{Operator Learning}
\subsubsection{Bubble Growth Dynamics-DeepONet}
\begin{figure}[H]
    \centering
    \includegraphics[width=1\linewidth]{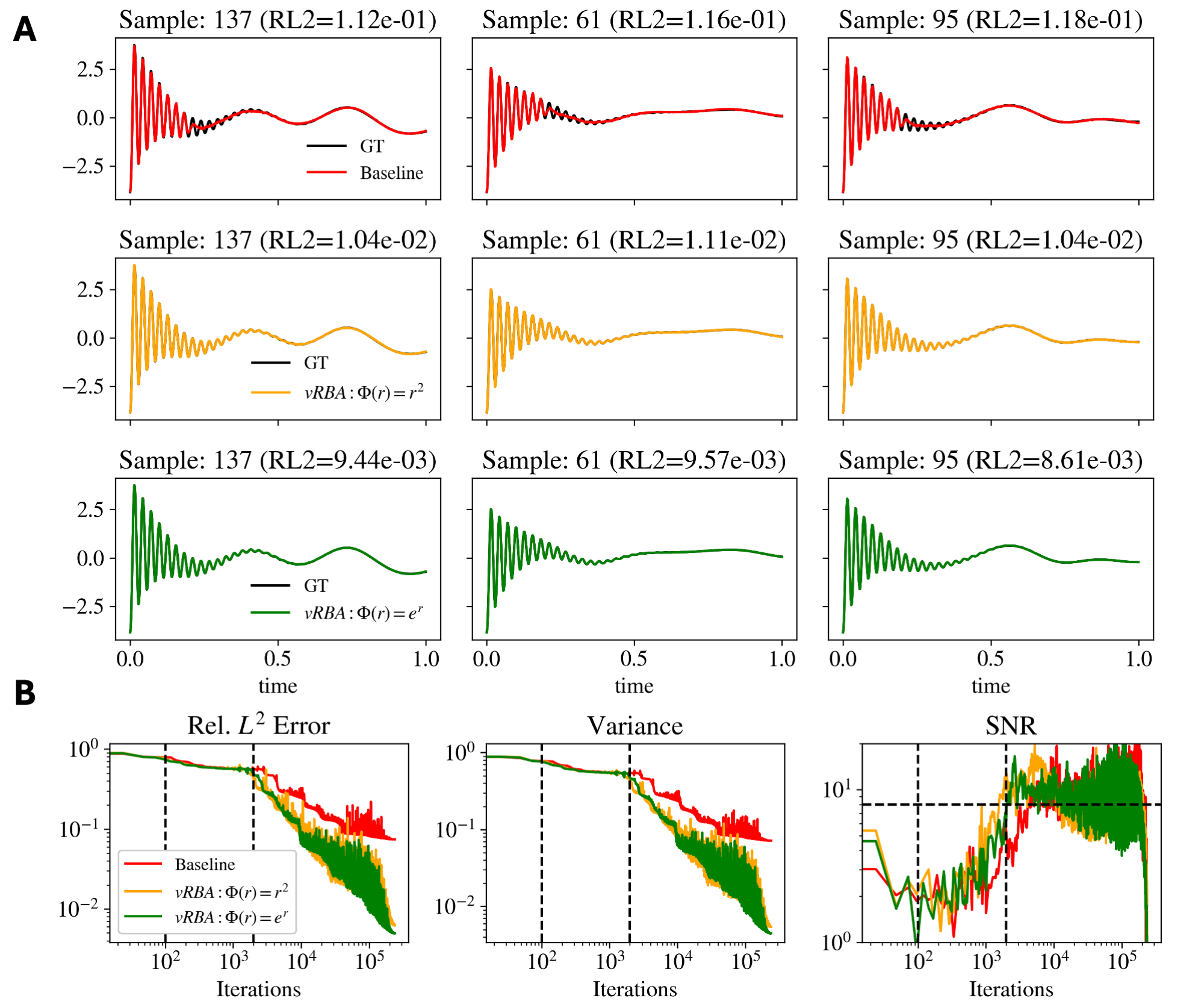}
\caption{
\textbf{Bubble Growth Dynamics with DeepONet.}
\textbf{(A)} Each column displays a different test function, corresponding to the three examples with the highest error for the baseline model. The rows compare the predictions of the baseline model (top), vRBA with a quadratic potential (middle), and vRBA with an exponential potential (bottom) against the ground truth (GT). The baseline model fails to capture the high-frequency oscillations in the bubble's dynamics, while both vRBA methods successfully reconstruct these fine-scale features.
\textbf{(B)} The plots show the convergence of relative $L^2$ error, residual variance, and SNR. The vRBA methods achieve a final error that is an order of magnitude lower than the baseline and reduce the variance by two orders of magnitude. This leads to faster and more stable training, evidenced by a consistently higher SNR. The dashed vertical lines mark the learning phases, showing that the vRBA models transition to the productive diffusion stage more rapidly.
}
    \label{DeepONet}
\end{figure}

We study the dynamics of a single gas bubble in an incompressible liquid governed by the Rayleigh–Plesset (R–P) equation \cite{lin2021operator}, a nonlinear ordinary differential equation describing the evolution of the bubble radius \( R(t) \) under a time-varying pressure field \( P_\infty(t) \). 
Under isothermal assumptions and negligible temperature variations, the simplified linearized R–P equation reads
\begin{equation}\label{eq:RP}
- \frac{\Delta p(t)}{\rho_L} = R_0 \frac{d^2 r}{dt^2} + \frac{4\nu_L}{R_0} \frac{dr}{dt} + \frac{1}{\rho_L R_0} \left(3 P_{G0} - \frac{2\gamma}{R_0} \right) r(t),
\end{equation}
where \( r(t) = R(t) - R_0 \) is the deviation from the initial bubble radius \( R_0 \), \( \rho_L \) is the liquid density, \( \nu_L \) is the kinematic viscosity, \( \gamma \) is the surface tension, and \( P_{G0} \) is the initial gas pressure inside the bubble.

We generate a dataset by numerically solving equation~\eqref{eq:RP} for 1000 independent realizations of the forcing function \( \Delta p(t) \), which is constructed as a product of a Gaussian random field and a smooth ramp function, following the procedure in \cite{lin2021operator}. Specifically, the pressure field is modeled as
\[
\Delta p(t) = g(t) s(t), \quad g(t) \sim \mathcal{GP}( \mu, \sigma^2 k(t_1, t_2) ),
\]
where \( k(t_1, t_2) \) is a squared exponential kernel with correlation length \( \ell \), and \( s(t) \) is a smooth ramp used to induce a sharp initial pressure drop.

The data were split into training, validation, and testing subsets in the ratio 80:10:10. Each simulation yields a trajectory of the bubble radius \( R(t) \), sampled over a fixed time window with initial condition \( R(0) = R_0 \), \( \dot{R}(0) = 0 \). All simulations assume periodic boundary conditions and are performed with parameters corresponding to physical properties of water at room temperature.

To predict the evolution of the bubble radius, we train a DeepONet to learn the mapping from the pressure profile $\Delta p(t)$ to the radius trajectory $R(t)$~\cite{lin2021operator}. For this and all other operator learning tasks, we apply vRBA using a hybrid strategy of importance weighting in the temporal domain and importance sampling over the function space.

The results, summarized in Table~\ref{all_Operators}, show a dramatic improvement. The vRBA framework reduces the final relative $L^2$ error by more than an order of magnitude, from $7.41 \times 10^{-2}$ for the baseline to $4.97 \times 10^{-3}$, with only a minor increase in computational cost per iteration. This quantitative gain is reflected in the qualitative predictions shown in Figure~\ref{DeepONet}(A). For challenging test cases, the baseline model fails to capture the high-frequency oscillations in the bubble's dynamics, whereas both vRBA methods successfully reconstruct these fine-scale features.

The learning dynamics in Figure~\ref{DeepONet}(B) explain this superior performance through the twofold mechanism observed in the PINNs examples. First, vRBA reduces the variance by two orders of magnitude, leading to more stable training. Second, it induces a consistently higher SNR and a much faster transition to the productive diffusion phase, indicating more effective learning.

\begin{table}[H]
    \centering
    \begin{tabular}{lcccc}
        \toprule
        \textbf{Problem}&\textbf{Sampling} & \textbf{N. Params} &  \textbf{Time (ms/it)} & \textbf{Rel. $L^2$ Error} \\
        \midrule
        BGD (DeepONet)&Baseline       & 101100 & 220 & $7.41 \times 10^{-2}$ \\
        &$vRBA:\Phi(x)=r^2$      & 101100   & 230  & $6.24 \times 10^{-3}$ \\
        &$vRBA:\Phi(x)=e^r$      & 101100& 230 & $\mathbf{4.97 \times 10^{-3}}$ \\
        \midrule
        NS (FNO)&Baseline       & 1622849 & 445 & $5.13 \times 10^{-2}$ \\
        &$vRBA:\Phi(x)=r^2$      & 1622849   & 459  & $2.37 \times 10^{-2}$ \\
        &$vRBA:\Phi(x)=e^r$      & 1622849 & 460 & $\mathbf{2.25\times 10^{-2}}$ \\
        \midrule
        WE (TC-UNet)&Baseline      & 2432001 & 810 & $3.69 \times 10^{-2}$ \\
        &$vRBA:\Phi(x)=r^2$      & 2432001   & 866  & $\mathbf{1.00 \times 10^{-2}}$ \\
        &$vRBA:\Phi(x)=e^r$      & 2432001 & 867 & $1.05\times 10^{-2}$ \\
        \bottomrule
    \end{tabular}
\caption{Performance of the vRBA framework on three operator learning benchmarks: bubble growth dynamics (BGD) solved with a DeepONet, the Navier-Stokes (NS) equations for Kolmogorov flow with an FNO, and the wave equation (WE) with a TC-UNet. In all cases, the vRBA method significantly outperforms the baseline model, which uses uniform sampling. The performance gain is most pronounced for the DeepONet architecture, where vRBA reduces the relative $L^2$ error by an order of magnitude, with significant improvements also observed for the FNO and TC-UNet models. For these operator learning tasks, all models were trained with the Adam optimizer, and the vRBA method was applied using a hybrid strategy of importance weighting in the spatial domain and importance sampling over the function space.}
\label{all_Operators}
\end{table}

\subsubsection{Navier Stokes-FNO}

We consider the two-dimensional unsteady Navier–Stokes equations in vorticity formulation, modeling an incompressible, viscous fluid on the periodic domain \((x, y) \in (0, 2\pi)^2\). The system is driven by a Kolmogorov-type external forcing, as previously studied in \cite{chandler2013invariant}, and is governed by:
\begin{equation}\label{eq:NS_new}
\begin{cases}
    \partial_t \omega + \bm{u} \cdot \nabla \omega = \nu \Delta \omega + f(x,y), \\
    \nabla \cdot \bm{u} = 0, \\
    \omega(x, y, 0) = \omega_0(x, y),
\end{cases}
\end{equation}
with viscosity \(\nu = 10^{-3}\), and the source term defined as 
\begin{align*}
    f(x,y) = \chi \left( \sin(2\pi(x + y)) + \cos(2\pi(x + y)) \right),
\end{align*}
where \(\chi = 0.1\). The Laplacian \(\Delta\) acts in two spatial dimensions, \(\omega\) denotes the vorticity, and $\bm{u}$ is the velocity. 

Initial conditions \(\omega_0(x, y)\) are sampled from a Gaussian random field with zero mean and covariance operator \(\mathcal{N}(0, 7^{3/2}(-\Delta + 49 I)^{-5/2})\). To generate the data, we employ a Fourier-based pseudo-spectral solver introduced in \cite{li2020fourier}. The simulation output consists of 1000 spatiotemporal vorticity realizations, each on a \(512 \times 512\) spatial grid, subsequently downsampled to \(128 \times 128\) for downstream learning tasks.

We partition the dataset into training, validation, and testing subsets in an 80:10:10 ratio. A neural operator model \(\mathcal{G}\) is trained to predict evolution of the vorticity field by learning the mapping from the initial condition at \(t=0\) to the interval \(t \in (0, 50]\).

\begin{figure}[H]
    \centering
    \includegraphics[width=0.95\linewidth]{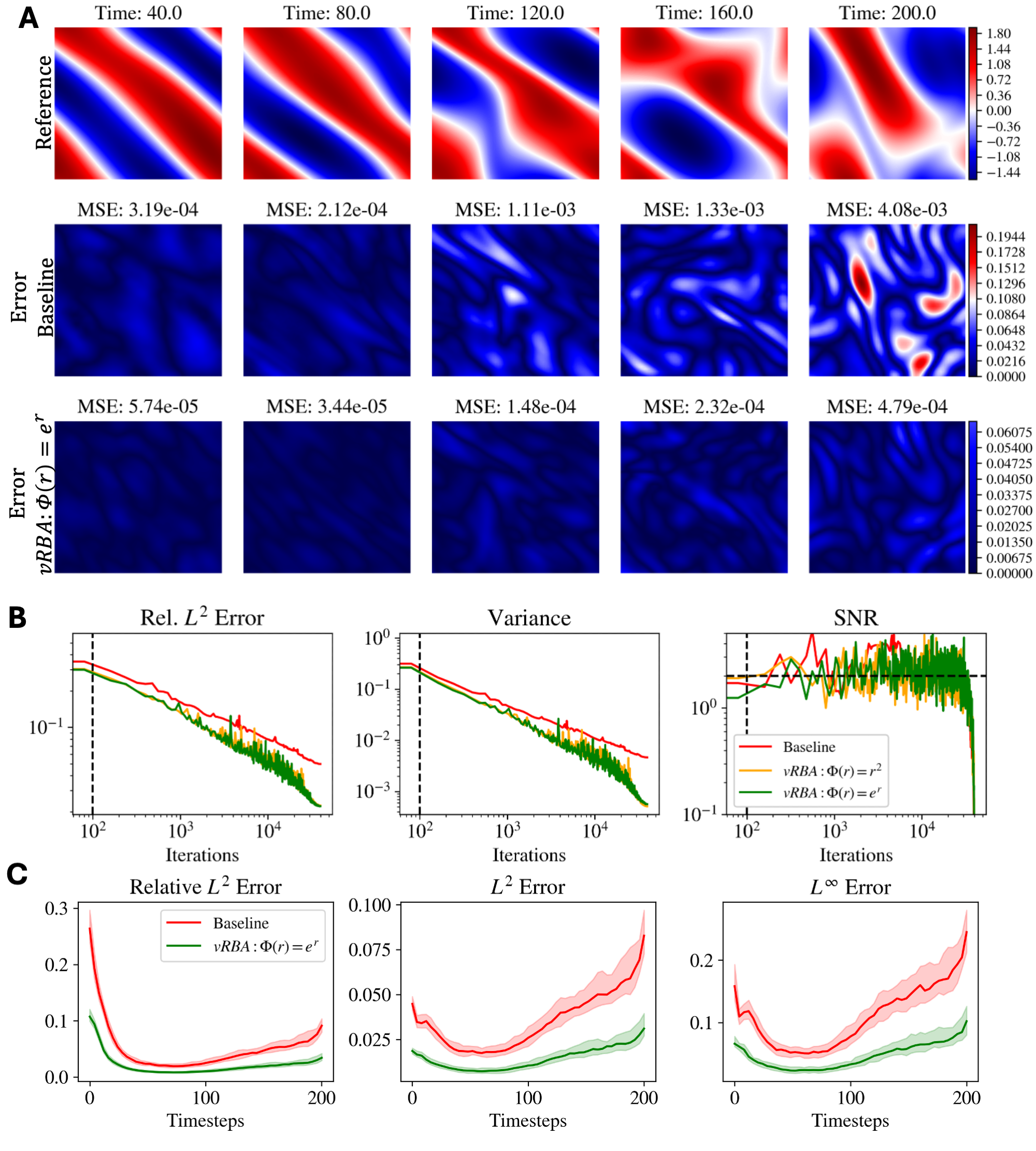}
\caption{
\textbf{FNO Performance on 2D Navier-Stokes (Kolmogorov Flow).}
\textbf{(A)} The rows display the reference vorticity field, the pointwise error of the baseline FNO, and the error of the vRBA-enhanced FNO at five temporal snapshots for a test trajectory. At each timestep, the error distribution is much more uniform, and the Mean Squared Error (MSE) is consistently lower than the baseline.
\textbf{(B)} Training dynamics showing the convergence of relative $L^2$ error, residual variance, and SNR. The vRBA methods achieve a lower final error and smaller variance. Interestingly, the SNR for all models remains high throughout training, suggesting the models start in or near the productive diffusion stage, a behavior possibly attributed to the FNO architecture's Fourier features.
\textbf{(C)} These panels show the mean error norms (solid lines) with standard deviation (shaded areas) evaluated over all test trajectories. The vRBA model exhibits both a lower mean error and a smaller standard deviation, indicating more robust and generalizable performance. Crucially, vRBA also shows a much slower rate of error accumulation over time.
}
    \label{FNO}
\end{figure}

For the 2D Navier-Stokes problem, we train a Fourier Neural Operator (FNO) to learn the mapping from an initial vorticity field $\omega_0(x, y)$ to the full spatiotemporal solution $\omega(x, y, t)$.

The vRBA framework again provides a significant performance boost, as shown in Table~\ref{all_Operators}. It more than halves the final relative $L^2$ error, reducing it from $5.13 \times 10^{-2}$ to $2.25 \times 10^{-2}$, with a negligible impact on the computational time per iteration. The qualitative results in Figure~\ref{FNO}(A) are even more striking, showing that the pointwise Mean Squared Error for the vRBA model is often nearly an order of magnitude lower than the baseline at different temporal snapshots.

The analysis of the error accumulation over the test set in Figure~\ref{FNO}(C) is particularly important. The vRBA model is not only more accurate on average (lower mean error) but also more robust and generalizable (smaller standard deviation). Most critically, it exhibits a significantly slower rate of error accumulation, a key advantage for long-term, stable predictions.

The training dynamics in Figure~\ref{FNO}(B) present a unique behavior. While vRBA still yields a lower final error and reduced variance, the SNR for all models starts and remains high throughout training. This suggests that the FNO architecture, likely due to the strong spectral bias from its built-in Fourier features, begins training in or near the productive diffusion stage, bypassing the typical fitting and transition phases observed in other architectures.

\subsubsection{Wave-Equation- TC-UNet}

We investigate the propagation of acoustic waves governed by the linear wave equation in heterogeneous media. 
In 2D, the governing equation is given by:
\begin{equation}\label{eq:wave_eq}
\begin{cases}
    \partial_t^2 u(\bm{x}, t) = c^2(\bm{x}) \Delta u(\bm{x}, t), & \bm{x} \in [0, \pi]^2, \ t \in [0, 2], \\
    u(\bm{x}, 0) = u_0(\bm{x}), \quad \partial_t u(\bm{x}, 0) = 0, & \bm{x} \in [0, \pi]^2,
\end{cases}
\end{equation}
where \(u(\bm{x}, t)\) represents the acoustic pressure at spatial location \(\bm{x} = (x, y)\), \(c(\bm{x})\) is the spatially varying wave speed, and \(\Delta\) denotes the Laplacian operator. We assume fully reflective (homogeneous Dirichlet) boundary conditions throughout the domain.

For the spatially varying wave speed, we set \(c(x, y) = 1 + \sin(x)\sin(y)\). The initial pressure profile \(u_0(\bm{x})\) is modeled as a localized Gaussian source centered at a point \(\bm{x}_c\), i.e.,
\[
u_0(\bm{x}) = \exp\left(-\frac{\|\bm{x} - \bm{x}_c\|^2}{10}\right),
\]
with \(\bm{x}_c\) sampled randomly on the spatial grid. We solve this system numerically using a second-order finite difference method on a grid with a spatial resolution of \(64 \times 64\) and generate 1000 simulations corresponding to different realizations of \(u_0\). The dataset is partitioned into training, validation, and test sets in the ratio 80:10:10. We train a neural operator \(\mathcal{G}\) to learn the mapping \(u(\bm{x}, 0) \mapsto u(\bm{x}, t)\) for all \(t \in (0, 2]\). 

In our final operator learning example, we train a Time-Conditioned U-Net (TC-UNet) to learn the solution operator for the 2D wave equation, mapping an initial pressure profile $u_0(\bm{x})$ to the full wave propagation over time $u(\bm{x}, t)$.

\begin{figure}[H]
    \centering
    \includegraphics[width=0.95\linewidth]{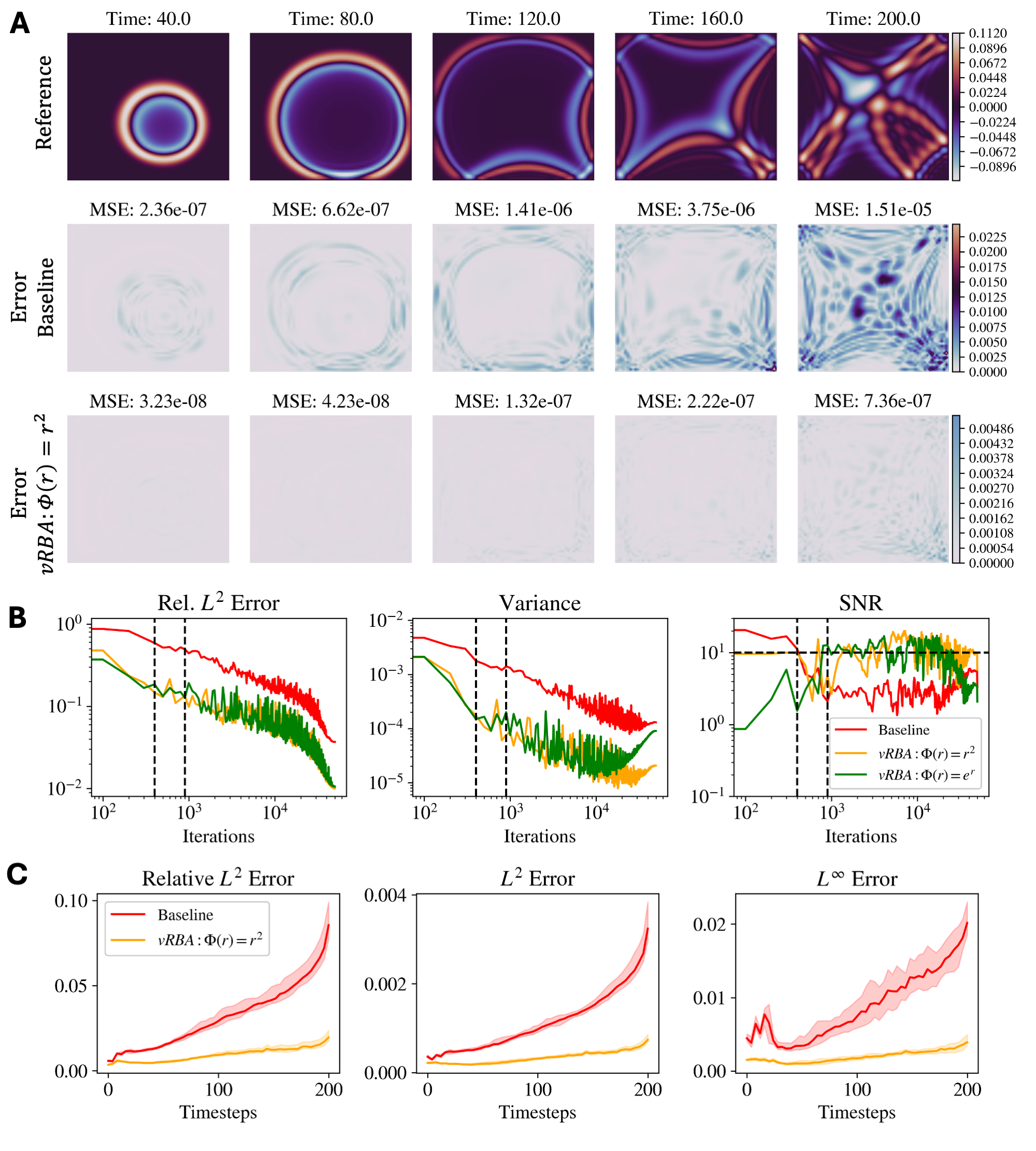}
\caption{
\textbf{TC-UNet Performance on the 2D Wave Equation.}
\textbf{(A)} The rows show the reference solution, baseline pointwise error, and vRBA pointwise error at five temporal snapshots for a representative trajectory from the test dataset. The vRBA method yields a more uniform error distribution and achieves lower Mean Squared Error (MSE).
\textbf{(B)} The plots for relative $L^2$ error, variance, and SNR demonstrate vRBA's superior performance during training. Both vRBA models achieve a lower final error and reduced variance, though the variance for the exponential potential increases late in training, possibly due to overfitting. The SNR plot shows the three phases of learning, with the vRBA methods transitioning to the diffusion phase faster.
\textbf{(C)} These plots show the mean error norms (solid lines) and standard deviation (shaded areas) evaluated over all trajectories in the test set. The vRBA $\Phi(r)=r^2$ model not only has a lower mean error but also a significantly smaller standard deviation, indicating more robust and generalizable performance. Furthermore, the baseline model's error accumulates at a much faster rate.
}    \label{TcUNet}
\end{figure}

As shown in Table~\ref{all_Operators}, vRBA again delivers a substantial improvement, reducing the final relative $L^2$ error by a factor of 3.7, from $3.69 \times 10^{-2}$ down to $1.00 \times 10^{-2}$. The qualitative results in Figure~\ref{TcUNet}(A) for a single test trajectory are particularly compelling, showing that the pointwise MSE for the vRBA model is over an order of magnitude smaller than the baseline, and the error is far more spatially uniform.

The analysis over the full test set in Figure~\ref{TcUNet}(C) confirms the method's effectiveness. The vRBA model is not only more accurate on average but also more robust, as indicated by the smaller standard deviation across all test cases. It also demonstrates a slower rate of error accumulation, a crucial property for predictive accuracy over long time horizons.

The training dynamics in Figure~\ref{TcUNet}(B) mirror the behavior seen in the DeepONet example, showcasing the twofold benefit of vRBA. First, it reduces the variance of the loss estimator, though we note a slight increase late in training for the exponential potential, possibly indicating the onset of overfitting. Second, it induces a higher SNR, allowing the model to transition from the fitting to the diffusion phase more rapidly, which leads to faster and more reliable convergence.

\section{Summary and Discussion}
\label{summary}
In this work, we addressed the largely heuristic nature of residual-based adaptive methods in scientific machine learning. We introduced a unifying variational framework that provides a formal justification and a principled design strategy for these techniques. By leveraging variational representations of integrated-convex functionals, we established a direct link between the form of the adaptive weights or sampling distribution and the primal optimization objective. This connection formally unifies residual-based attention weights (RBA) and residual-based adaptive distribution (RAD) schemes, showing they are different practical implementations of the same underlying principle: optimizing a dual formulation of a chosen primal objective.

The benefits of this framework are manifold and explain the large success of residual-based attention methods from previous studies. It provides a principled origin for heuristics by showing that the choice of potential function dictates the primal optimization objective, such as an exponential potential for $L^\infty$ minimization or a quadratic potential for variance reduction. By promoting a more uniform magnitude for the residuals, vRBA forces the model to capture fine solution details often missed during standard optimization. Furthermore, the framework directly addresses discretization error by reducing the variance of the loss estimator, a benefit we numerically demonstrated to be several orders of magnitude. Moreover, vRBA enhances the learning dynamics; our analysis reveals that vRBA maintains a high signal-to-noise ratio (SNR) of the back-propagated gradients throughout training and transitions to diffusion faster, thereby accelerating convergence.

We demonstrated the efficacy and versatility of vRBA across a range of challenging benchmarks in both PINNs and, notably, operator learning. For the latter, we introduced a hybrid strategy, employing importance sampling over the function space and importance weighting over the spatial domain, which proved particularly effective at reducing the rate of error accumulation. Our empirical results confirmed that vRBA is a critical component for achieving high accuracy, providing substantial improvements even when paired with state-of-the-art second-order optimizers and diverse neural operator architectures.

This work provides a formal basis for the success of residual-based adaptation and opens new avenues for the principled design of effective discretization and optimization strategies. The implications of our variational perspective extend beyond methods that use the residual directly. For instance, the learnable weights proposed in self-adaptive methods, whether treated as trainable parameters or generated by auxiliary networks~\citep{mcclenny2023self, zhang2023dasa}, can be reinterpreted through our framework as a strategy for learning the optimal biasing distribution. This can be viewed as an alternative approach to solving the dual optimization problem, where the distribution is learned rather than derived analytically from a fixed potential. Our framework thus provides a theoretical lens through which an even broader class of adaptive techniques can be unified and analyzed.

\section*{Acknowledgements}
We acknowledge the support of the NIH grant R01AT012312, MURI/AFOSR FA9550-20-1-0358 project, the DOE-MMICS SEA-CROGS DE-SC0023191 award, and the ONR Vannevar Bush Faculty Fellowship (N00014-22-1-2795).

\section*{Data availability}

To support reproducibility, the source code for our implementation and data will be publicly available in our \href{https://github.com/jdtoscano94/Variational-Residual-Based-Attention-vRBA-for-PINNs-and-Operator-Networks-.git}{GitHub repository} upon acceptance of the manuscript.

\bibliographystyle{elsarticle-num} 
\bibliography{cas-refs}
\newpage
\appendix

\section{Representation of divergences}
\label{sec:theory}

The derivation of vRBA uses several variational formulas of statistical divergences and related functionals.
The goal of this appendix is to review the various representations used to derive vRBA with notation consistent with the manuscript.

The first is the Laplace principle, which rewrites the maximum as integration against increasingly singular measures.
While the result is well-known in the context of large deviations, e.g., \cite[Theorem 4.3.1]{dembo2009large} or \cite[Theorem 1.5]{budhiraja2019analysis}, the statement as presented is difficult to find, so we provide the proof for completeness.

\begin{proposition} \label{prop:laplace-principle}
Let $r: \Omega \to \R$ be a bounded, measurable function and let $p \in \mathcal P(\Omega)$ be a probability measure.
Let
\begin{align*}
    M = \inf \left\{ m \in \R: p \left( x \in \Omega: r(x) \leq m \right) = 1 \right\}
\end{align*}
be the essential supremum of $r$.
Then, we have
\begin{align*}
    M = \sup_{\epsilon > 0} \epsilon \log \int_\Omega e^{r(x)/\epsilon} p(dx) = \lim_{\epsilon \to 0} \epsilon \log \int_\Omega e^{r(x)/\epsilon} p(dx).
\end{align*}
\end{proposition}

\begin{proof}
First, observe that 
\begin{align*}
    \epsilon \log \int_\Omega e^{r(x)/\epsilon} p(dx) \leq \epsilon \log \int_\Omega e^{M/\epsilon} p(dx) = M
\end{align*}
for each $\epsilon > 0$.
Therefore, the problem reduces to proving a matching lower bound.
We first prove the statement for the case where $p$ is supported on finitely many points, then proceed via approximation.

\paragraph{\underline{Finite support}}
Without loss of generality, we can reduce the problem to showing for $a, b \in \R$ and $a > b$,
\begin{align*}
    \lim_{\epsilon \to 0} \epsilon \log \left( e^{a/\epsilon} + e^{b/\epsilon} \right) = a.
\end{align*}
This follows from rewriting
\begin{align*}
    \epsilon \log \left( e^{a/\epsilon} + e^{b/\epsilon} \right) = \epsilon \log \left( e^{a/\epsilon} \left( 1 + e^{(b - a)/\epsilon} \right) \right) = a + \epsilon \log \left( 1 + e^{(b - a)/\epsilon} \right)
\end{align*}
and the convergence follows from $b - a < 0$.

\paragraph{\underline{Approximation}}
Without loss of generality, shift and scale $r$ so that $0 \leq r \leq 1$ and $M = 1$.
Fix $k \in \mathbb N$ and define a partition (up to measure-zero sets)
\begin{align*}
    \Omega = \bigcup_{m=0}^{k-1} A_m ~~\text{where}~~A_m = \left\{ x \in \Omega: \frac{m}{k} \leq r(x) \leq \frac{m+1}{k} \right\}.
\end{align*}
Then, we can partition the integral:
\begin{align*}
    \epsilon \log \int_\Omega e^{r(x)/\epsilon} p(dx) &= \epsilon \log \left( \sum_{m=0}^{k-1} \int_{A_m} e^{r(x) / \epsilon} p(dx) \right) \\
    &\geq \epsilon \log \left( \sum_{m=0}^{k-1} p(A_m) e^{m / k \epsilon} \right) \\
    &\to \left( 1 - \frac{1}{k} \right) ~~\text{as}~~\epsilon \to 0,
\end{align*}
where the last equality is by the convergence in the finite-support case.
Finally, taking increasing fine partitions, i.e., sending $k \to \infty$, completes the proof.
\end{proof}

Complementary to the Laplace principle, the Gibbs variational formula gives a representation for the log-integrated-exponentials---also known as the log-partition function.

\begin{proposition}\label{prop:gibbs-variational}
Let $r: \Omega \to \R$ be bounded and continuous and $p \in \mathcal P(\Omega)$ be a probability measure. 
Then, the variational formula holds: 
\begin{align*}
    \log \int_\Omega e^{r(x)} p(dx) = \sup_{q \in \mathcal P(\Omega)} \left\{ \int_\Omega r(x) q(dx) - \mathbf H(q|p) \right\}
\end{align*}
where $\mathbf H$ denotes the relative entropy (defined in \eqref{eq:def-rel-ent}).
Moreover, the optimizer $q^* \in \mathcal P(\Omega)$ that achieves the supremum takes the form
\begin{align*}
    q^*(dx) = \frac{e^{r(x)}}{\int_\Omega e^{r(x)} p(dx)} p(dx).
\end{align*}
\end{proposition}

A proof can be found in, for example, \cite[Proposition 2.2]{budhiraja2019analysis}.
It consists of two steps: first, showing that the identified $q^*$ achieves equality.
Then, decompose the log-likelihood-ratio into ones with respect to the optimizer $q^*$ allows one to use the non-negativity of relative entropy to conclude.

In fact, the Gibbs variational principle is only one part of a duality with the log-partition function.
This is known as the Donsker-Varadhan representation, e.g., \cite[Lemma 2.4]{budhiraja2019analysis}, which gives a representation of the relative entropy as the Legendre-Fenchel transform of the log-partition function:
\begin{align*}
    \mathbf H(p | q) = \sup_{r \in \mathcal M_b(\Omega; \R)} \left\{ \int_\Omega r(x) q(dx) - \log \int_\Omega e^{r(x)} p(dx) \right\}.
\end{align*}
The set $\mathcal M_b(\Omega; \R)$ denotes the set of bounded measurable functionals $r: \Omega \to \R$.
A similar duality theory was established in \cite{birrell2022f}.

\begin{proposition}\label{prop:generalized-gibbs}
Let $r: \Omega \to \R$ be bounded and measurable and let $p \in \mathcal P(\Omega)$ be a probability measure.
Let $\Phi: \R \to \R$ be convex, bounded from below, and superlinear.
Then, we have:
\begin{enumerate}
    \item a generalized Gibbs variational formula: 
    \begin{align*}
        \inf_{\nu \in \R} \left\{ \nu + \int_\Omega \Phi(r(x) - \nu) p(dx) \right\} = \sup_{q \in \mathcal P(\Omega)} \left\{ \int_\Omega r(x) q(dx) + \mathbf D_{\Phi^*}(q|p) \right\};
    \end{align*}
    \item when $r$ additionally satisfies
    \begin{align*}
        \int_\Omega \Phi'(r(x)) p(dx) = 1
    \end{align*}
    where $\Phi'$ is any element in the subdifferential of $\Phi$, then the optimal $\nu$ is zero and
    \begin{align*}
        \int_\Omega \Phi(r(x)) p(dx) = \sup_{q \in \mathcal P(\Omega)} \left\{ \int_\Omega r(x) q(dx) + \mathbf D_{\Phi^*}(q|p) \right\}
    \end{align*}
    with equality achieved when
    \begin{align*}
        q(dx) = \Phi'(r(x)) p(dx).
    \end{align*}
\end{enumerate}
\end{proposition}

\begin{proof}
The first item follows from \cite[Corollary 58]{birrell2022f}.
The proof is involved and will not be reproduced here.

The second item follows first from a Donsker-Varadhan-type representation of $\Phi$-divergences, which can be found in, e.g., \cite[Lemma 1]{nguyen2010estimating}.
By Fenchel-Young, 
\begin{align*}
    \mathbf D_{\Phi^*}(q|p) = \int_\Omega \Phi^* \left( \frac{dq}{dp}(x) \right) p(dx) \geq \int_\Omega \Phi \left( r(x) \right) p(dx) + \int_\Omega r(x) q(dx)
\end{align*}
for any bounded, measurable $r: \Omega \to \R$.
Moreover, equality is achieved when for each $x \in \Omega$,
\begin{align*}
    r(x) \in \partial \Phi^* \left( \frac{dq}{dp}(x) \right) ~~\Longleftrightarrow~~ \frac{dq}{dp}(x) \in \partial \Phi(r(x));
\end{align*}
$\partial \Phi$ refers to the subdifferential.
Rearranging yields the desired variational form with the proposed optimizer assuming the normalization condition holds.
\end{proof}

The optimizer for the generalized Gibbs variational formula is generally not identified, necessitating the normalization condition.

Finally, we turn to the functional $\sup_{\epsilon > 0} \Lambda_\epsilon$ defined in \eqref{eq:p1-phi-lambda}.
Unlike in the Laplace case, i.e., $\Phi(r) = e^r -r + 1$, the objective of the primal minimization problem is not transparent.
Below, we show that for the standard quadratic loss corresponds to the primal objective of minimizing variance.
In particular, we choose $\Phi(r) = r^2 + 1$ to correspond to having chi-squared regularization in the dual problem, though the specific choice of constant does not matter.

\begin{proposition}\label{prop:quadratic-potential}
Let $r: \Omega \to \R$ be bounded, measurable and let $p \in \mathcal P(\Omega)$ be a probability measure.
Let $\Phi(r) = r^2 + 1$.
Then, 
\begin{align*}
    \sup_{\epsilon > 0} \Lambda_\epsilon(r) = \sqrt{\mathbb E_p[r^2] - \mathbb E_p[r]^2}.
\end{align*}
\end{proposition}

\begin{proof}
First, observe that, for each fixed $\nu \in \R$, the map
\begin{align*}
    \epsilon \mapsto \epsilon \Phi^{-1} \left( \frac{\nu}{\epsilon} + \mathbb E_p \left[ \Phi \left( \frac{r-\nu}{\epsilon} \right) \right] \right)
\end{align*}
monotonically increases as $\epsilon \to 0$.
This is by the monotonicity of $\Phi$, $\Phi^*$, and $\epsilon \mapsto 1/\epsilon$ as well as the non-negativity of $\epsilon$.
Thus, we can replace the supremum by a limit as $\epsilon \to 0$.

Now, we wish to compute 
\begin{align*}
    \lim_{\epsilon \to 0} \inf_{\nu \in \R} \left\{ \epsilon \Phi^{-1} \left( \frac{\nu}{\epsilon} + \mathbb E_p \left[ \Phi \left( \frac{r-\nu}{\epsilon} \right) \right] \right) \right\}
\end{align*}
by switching the limit and infimum.
For any $\delta > 0$, let $(\nu^\epsilon)_{\epsilon > 0}$ be a collection of $\delta$-minimizers of the inner infimum, that is, 
\begin{align*}
    \epsilon \Phi^{-1} \left( \frac{\nu^\epsilon}{\epsilon} + \mathbb E_p \left[ \Phi \left( \frac{r-\nu^\epsilon}{\epsilon} \right) \right] \right)  - \delta \leq \inf_{\nu \in \R} \left\{ \epsilon \Phi^{-1} \left( \frac{\nu}{\epsilon} + \mathbb E_p \left[ \Phi \left( \frac{r-\nu}{\epsilon} \right) \right] \right) \right\}.
\end{align*}
Moreover, the map 
\begin{align*}
    \nu \mapsto \epsilon \Phi^{-1} \left( \frac{\nu}{\epsilon} + \mathbb E_p \left[ \Phi \left( \frac{r-\nu}{\epsilon} \right) \right] \right)
\end{align*}
is coercive by the superlinearity of $\Phi$ and uniformly so in $\epsilon$.
Thus, $(\nu^\epsilon)_{\epsilon > 0}$ is precompact and has a subsequential limit point, which we denote by $\nu^*$.
Now, we have that
\begin{align*}
    \liminf_{\epsilon \to 0} \inf_{\nu \in \R} \left\{ \epsilon \Phi^{-1} \left( \frac{\nu}{\epsilon} + \mathbb E_p \left[ \Phi \left( \frac{r-\nu}{\epsilon} \right) \right] \right) \right\} &\geq \liminf_{\epsilon \to 0} \epsilon \Phi^{-1} \left( \frac{\nu^\epsilon}{\epsilon} + \mathbb E_p \left[ \Phi \left( \frac{r-\nu^\epsilon}{\epsilon} \right) \right] \right)  - \delta \\
    &= \liminf_{\epsilon \to 0} \Phi^{-1} \left( \epsilon \nu^\epsilon + \mathbb E_p \left[ \Phi \left( r-\nu^\epsilon\right) \right] \right)  - \delta \\
    &= \liminf_{\epsilon \to 0} \epsilon \Phi^{-1} \left( \mathbb E_p \left[ \Phi \left( r-\nu^* \right) \right] \right)  - \delta \\
    &= \Phi^{-1} \left( \inf_{\nu \in \R} \mathbb E_p[\Phi(r-\nu)] \right) - \delta.
\end{align*}
On the other hand, we have straightforwardly by comparison that
\begin{align*}
    \lim_{\epsilon \to 0} \inf_{\nu \in \R} \left\{ \epsilon \Phi^{-1} \left( \frac{\nu}{\epsilon} + \mathbb E_p \left[ \Phi \left( \frac{r-\nu}{\epsilon} \right) \right] \right) \right\} &\leq \inf_{\nu \in \R} \lim_{\epsilon \to 0} \left\{ \epsilon \Phi^{-1} \left( \frac{\nu}{\epsilon} + \mathbb E_p \left[ \Phi \left( \frac{r-\nu}{\epsilon} \right) \right] \right) \right\} \\
    &= \Phi^{-1} \left( \inf_{\nu \in \R} \mathbb E_p[\Phi(r-\nu)] \right).
\end{align*}
Therefore, sending $\delta \to 0$ shows that the limit and infimum can be interchanged.

Finally, we solve the variational problem by checking first-order optimality.
The problem is now convex, and straightforward computation yields that the optimal $\nu$ is achieved by choosing the mean.
This concludes the proof.
\end{proof}
\section{Physics Informed Machine Learning}
\subsection{Global weights}
Notice that for first-order optimizers such as ADAM, the update direction (i.e., equation \eqref{update_discrete}) for PINNs (i.e., equation \eqref{PIML_loss}) is given by:
\begin{align}
    p^k &= -m_E\nabla_{\theta}\mathcal{L}_E(\theta^{k}) - m_B\nabla_{\theta}\mathcal{L}_B(\theta^{k}) - m_D\nabla_{\theta}\mathcal{L}_D(\theta^{k}),
\end{align}
where $\nabla_{\theta}\mathcal{L}_{E}$, $\nabla_{\theta}\mathcal{L}_{B}$, and $\nabla_{\theta}\mathcal{L}_{D}$ are the loss gradients which can be represented as high-dimensional vectors defining directions to minimize their respective loss terms. Notice that if the gradient magnitudes are imbalanced, one direction will dominate, which may lead to poor convergence.  To address this challenge, we propose modifying the magnitude of the individual directions by scaling their respective global weights. In particular, we fix $m_E$ and update the remaining global weights using the rule:
\begin{align}
\label{boundary_weight}
    m_B^k &= \alpha m_B^{k-1} + (1-\alpha) \frac{\|\nabla_{\theta}\mathcal{L}_E\|}{\|\nabla_{\theta}\mathcal{L}_B\|}, \\
\label{data_weight}
    m_D^k &= \alpha m_D^{k-1} + (1-\alpha) \frac{\|\nabla_{\theta}\mathcal{L}_E\|}{\|\nabla_{\theta}\mathcal{L}_D\|},
\end{align}
where $\alpha \in [0,1]$ is a stabilization parameter~\cite{wang2021understanding}. This formulation computes the iteration-wise average ratio between gradients, enabling normalized scaling, which, on average, allows us to define a balanced update direction $\hat{p}^k$:
\begin{align}
\label{update_dir_grads}
    \hat{p}^k &\approx -m_E\|\nabla_{\theta}\mathcal{L}_E\| \left[ \nabla_{\theta}\mathcal{L}_E(\theta^{k}) - \frac{\nabla_{\theta}\mathcal{L}_B(\theta^{k})}{\|\nabla_{\theta}\mathcal{L}_B\|} - \frac{\nabla_{\theta}\mathcal{L}_D(\theta^{k})}{\|\nabla_{\theta}\mathcal{L}_D\|} \right].
\end{align}

Under this approach, all loss components have balanced magnitudes, allowing each optimization step to minimize all terms effectively.

\subsection{Algorithm for PINNs}

\begin{algorithm}[H]
\small
  \caption{$vRBA$ for PINNs}
  \label{vRBA_alg}
  \textbf{Input:}
  Representation model: $\mathcal{M}$\\
  Training points: $X_B$,$X_D$,$X_E$\\
  Optimizer parameters: $lr$ 
  vRBA parameters: $\eta$, $\lambda_{max_0}$,$\lambda_{cap}$, $\alpha_g$, $m_E$, $\gamma_{g}$\\
  Number of iterations per stage: $N_{stage}$\\
  Total number of training of iterations: $N_{train}$\\
  Boolean flags: \textit{adaptive weights}, \textit{adaptive distribution}
  
 \textbf{Output:} 
 Optimized network parameters $\theta$
  \begin{algorithmic}[1]
  \STATE Initialize the network parameters: $\theta$ 
    \STATE Initialize RBA: $\lambda^0_{\alpha,i}=0.1\lambda_{max0}$  $\forall \alpha,i$  with $\alpha=\{B,D,E\}$\\
    \STATE Initialize uniform distribution: $\bar q_{\alpha}^{k}$  with $\alpha=\{B,D,E\}$\\
  \FOR{$k<N_{train}$}
  \STATE Update maximum RBA upper bound: $\lambda_{max}=\min(\lambda_{max0}+k/N_{stage},\lambda_{cap})$
  \STATE Update decay rate: $\gamma^k=1-\eta/\lambda_{max}$
  \FOR{each $\alpha\in\{B,D,E\}$}
  \IF{\textit{adaptive distribution}}
  \STATE Update the sampling p.m.f: $   \bar q_{\alpha}^{k}\leftarrow(\bm{\lambda}_{\alpha}^{k})/{\sum(\bm{\lambda}_{\alpha}^{k})}$
  \ENDIF    
  \STATE Sample $bs$ points from $X_{\alpha}$: $X^k_{\alpha}\sim \bar q_{\alpha}^{k}$
  \STATE Compute network prediction: $u_{\alpha,i}\leftarrow\mathcal{M}(\theta, x^k_{\alpha,i})$, $\forall x_{\alpha,i}\in X^k_{\alpha} $
  \STATE Compute residuals: $r_{\alpha,i}^{k}$ using $u_{\alpha,i}$ and equations~\ref{F_fit_loss}, or ~\ref{PDE_Loss}.
      \STATE Update tilted distribution: $q_{\alpha,i}^{k}$ using equation~\ref{exp_dist} or \ref{quad_dist}
      \STATE Exponential moving average: $\lambda_{\alpha,i}^{k} \leftarrow \gamma^k\lambda_{\alpha,i}^{k-1}+\eta^*q_{\alpha,i}^{k}$
      \IF{\textit{adaptive weights}}
      \STATE Compute loss term: $\mathcal{L}^k_{\alpha}=\langle(\lambda_{\alpha,i}^{k}r_{\alpha,i}^{k})^2\rangle$
      \ELSE
      \STATE Compute loss term: $\mathcal{L}^k_{\alpha}=\langle(r_{\alpha,i}^{k})^2\rangle$
      \ENDIF

      \STATE Compute gradient:$\nabla_{\theta}\mathcal{L}^k_{\alpha}$
      \STATE Compute the average gradient magnitude :$\|\nabla_{\theta}\bar{\mathcal{L}}^k_{\alpha}\|=\gamma_g\|\nabla_{\theta}\bar{\mathcal{L}}^{k-1}_{\alpha}\|+(1-\gamma_g)\|\nabla_{\theta}\mathcal{L}^{k-1}_{\alpha}\|$
    \ENDFOR
    \STATE Update data global weight: $m_D^{k}=\alpha_g m_D^{k-1}+(1-\alpha_g)m_E\|\nabla_{\theta}\bar{\mathcal{L}_E}^k\|/{\|\nabla_{\theta}\bar{\mathcal{L}_D}^k\|}$
    \STATE Define total update direction: $     p^k\leftarrow-m_E\nabla_{\theta}\mathcal{L}_E^{k}-m_D^k\nabla_{\theta}\mathcal{L}_D^{k}$
    \STATE Update parameters: $\theta^{k+1}\leftarrow\theta^{k}+lr^kp^k$
  \ENDFOR
  \end{algorithmic}
\end{algorithm}

\section{Operator Learning}
\label{Operators_descriptions}
Operator learning aims to approximate mappings between infinite-dimensional function spaces, inspired by the universal approximation theorem for nonlinear operators introduced by Chen and Chen \cite{chen1995universal}. 
In contrast to traditional supervised learning, which seeks point-wise mappings, operator learning targets functional input–output relationships, such as solution operators of PDEs. 
Operator learning is a suitable approach for problems where solutions must be inferred across varying initial or boundary conditions, enabling fast inference once trained. 
In this study, we consider DeepONet, Fourier Neural Operator (FNO), and Time-Conditioned UNet (TC-UNet) based architectures.

\subsection{DeepONet}
DeepONet consists of two networks - a trunk network and a branch network. 
The trunk network encodes spatial coordinates and learns a basis in the target function space, while the branch network maps the input function, evaluated at a fixed set of sensors, to coefficients that project onto this learned basis. 
The resulting dot product yields the output function at each spatial location. 
This design is rooted in the operator approximation theorem and enables expressive and efficient modeling of nonlinear operators. 
DeepONet and its variants are widely applied in mechanics \cite{kiyani2025predicting}, high-speed flows \cite{peyvan2024riemannonets}, materials science \cite{oommen2022learning} and multi-phase flows \cite{lin2021seamless}. 

\subsection{FNO}

FNO learn solution operators by leveraging spectral convolutions in the Fourier domain. 
The input function is first lifted to a high-dimensional latent space through pointwise linear transformations. 
A Fourier transform is applied to these lifted features, enabling convolutional operations to be performed as multiplications in frequency space. 
High-frequency modes are typically truncated to enforce smoothness, reduce overfitting, and improve training dynamics. 
The result is then transformed back to physical space via the inverse Fourier transform and projected to the target dimension. 
The global receptive field of FNOs makes them particularly effective for modeling long-range dependencies in solutions to PDEs, as demonstrated in applications such as weather forecasts \cite{kurth2023fourcastnet}, porous media flows \cite{kashefi2024novel}, and turbulence \cite{li2022fourier}.

\subsection{TC-UNet}

Unlike FNOs, TC-UNet \cite{ovadia2025real, gupta2022towards} operates entirely in physical space using local convolutions. 
The architecture is based on a UNet, a hierarchical fully convolutional neural network that captures multiscale features through successive downsampling and upsampling. 
TC-UNet uses time conditioning via feature-wise linear modulation (FiLM) \cite{perez2018film}, applied at each level of the hierarchy. 
This allows the model to adaptively modulate intermediate features based on the time coordinate input, enabling accurate modeling of spatiotemporal dynamics.
TC-UNet or UNet-based architectures are particularly well-suited for problems characterized by sharp gradients \cite{peyvan2024riemannonets} or fine-scale structures \cite{khodakarami2025mitigating} and are, in general, more robust to spectral bias \cite{oommen2024integrating, oommen2025equilibrium} compared to other neural operator architectures.
\subsection{Algorithm for Operator Learning}
\begin{algorithm}[H]
\small
\caption{$vRBA$ for Operator Learning}
\label{vRBA_NO_alg}
\textbf{Input:} \\
\quad Representation model (Neural Operator): $G_\theta$\\
\quad Training data: $\{v_j, u_j\}_{j=1}^{N_{func}}$ (a set of input/output function pairs) \\
\quad Spatial points per function: $N$\\
\quad Optimizer parameters: $lr$ \\
\quad vRBA parameters: $\eta, \lambda_{max_0}, \lambda_{cap}, \gamma$ \\
\quad Batch size for functions: $b_u$ \\
\quad Number of iterations per stage: $N_{stage}$\\
\quad Frequency of distribution update: $N_{update}$ \\
\quad Total number of training iterations: $N_{train}$
 
\textbf{Output:} Optimized network parameters $\theta$

\begin{algorithmic}[1]
\STATE Initialize the network parameters: $\theta^0$
\STATE Initialize weights: $\Lambda^0_{i,j} = 0.1\lambda_{max0}$ for $i=1..N, j=1..N_{func}$
\STATE Initialize function sampling p.m.f. uniformly: $\bar{q}_j^0 = 1/N_{func}$ for $j=1..N_{func}$

\FOR{$k=0, 1, \dots, N_{train}-1$}
    \STATE Update maximum RBA upper bound: $\lambda_{max} = \min(\lambda_{cap}, \lambda_{max0} + k/N_{stage})$
    \STATE Update EMA decay rate: $\gamma^k = 1 - \eta/\lambda_{max}$
    
    \STATE Sample a batch of $b_u$ function indices $\mathcal{J}_k \sim \bar{q}^k$
    
    \STATE Compute residuals for the batch: $R_{i,j}^k = G_{\theta^k}(v_j)(x_i) - u_j(x_i)$ for $i \in \{1..N\}, j \in \mathcal{J}_k$
    \STATE Update target distribution matrix for the batch: $Q_{i,j}^{k+1}$ using $|R_{i,j}^k|$ (via Eq.~\ref{exp_dist} or \ref{quad_dist})
    \STATE Update weights for the batch via EMA: $\Lambda_{i,j}^{k+1} \leftarrow \gamma^k \Lambda_{i,j}^{k} + \eta^* Q_{i,j}^{k+1}$ for $j \in \mathcal{J}_k$
    
    \STATE Compute the weighted loss for the batch: 
    $\mathcal{L}^k = \frac{1}{b_u N} \sum_{j \in \mathcal{J}_k} \sum_{i=1}^{N} [\Lambda_{i,j}^{k+1} R_{i,j}^k]^2$
    
    \STATE Compute gradient of the loss: $g^k = \nabla_{\theta} \mathcal{L}^k|_{\theta=\theta^k}$
    \STATE Update parameters: $\theta^{k+1} \leftarrow \theta^k - lr^k g^k$
    
    \IF{$k \pmod{N_{update}} == 0$}
        \STATE Aggregate importance scores for all functions: $s_j^{k+1} = \sum_{i=1}^{N} \Lambda_{i,j}^{k+1}$ for $j=1..N_{func}$
        \STATE Normalize scores to form new p.m.f.: $\bar{q}_j^{k+1} = s_j^{k+1} / \sum_{\ell=1}^{N_{func}} s_\ell^{k+1}$
    \ENDIF

\ENDFOR
\end{algorithmic}
\end{algorithm}

\section{Implementation Details}

\subsection{Physics-Informed Neural Networks}
For our Physics-Informed Neural Network (PINN) benchmarks, we detail two separate experimental setups based on the optimization strategy employed.

\subsubsection{First-Order Optimization}
For the Allen-Cahn equation solved with a first-order optimizer, the network architecture and hyperparameters were adapted from previous work in \cite{toscano2025kkans}. The specific implementation details are summarized in Table \ref{tab:ac_first_order}. This setup utilizes the Adam optimizer for the entire training duration and applies vRBA as an importance weighting scheme.

\begin{table}[H]
\centering
\begin{tabular}{|l|c|}
\hline
\textbf{Hyperparameter} & \textbf{Allen-Cahn} \\
\hline
N. of Adam training iterations & 3e5 \\
N. of SSBroyden training iterations & 0 \\
Number of hidden layers $N$ & 6 \\
Hidden layer dimension $H$ & 64 \\
Activation function & $\tanh(\cdot)$ \\
Fourier Feature embedding degree~\cite{wang2021eigenvector} & 10 \\
Initialization & $U\left(-\sqrt{\frac{3}{I}}, \sqrt{\frac{3}{I}}\right)$ \\
Learning rate $lr$ & 1e-3 \\
$lr$-Decay rate & 0.9 \\
$lr$-Decay step & 5000 \\
Total number of points & 2.56e4 \\
Batch size & 1e4 \\
\hline
\multicolumn{2}{|c|}{\textbf{vRBA Parameters (Weighting)}} \\
\hline
$\gamma$ (EMA memory) & 0.999 \\
$\eta$ (EMA learning rate) & 0.01 \\
$\phi$ (Smoothing) & 0.8 \\
\hline
\multicolumn{2}{|c|}{\textbf{Self-Scaling Parameters}} \\
\hline
$\lambda_{max0}$ (Initial max weight) & 10 \\
$\lambda_{cap}$ (Weight cap) & 20 \\
$\gamma_g$ (Gradient EMA memory) & 0.99 \\
$\alpha_g$ (Global weight EMA memory) & 0.99975 \\
$N_{stage}$ (Iterations per stage) & 50000 \\
$m_E$ (Equation loss weight) & 1.0 \\
\hline
\end{tabular}
\caption{Implementation details for the Allen-Cahn equation using a first-order Adam optimizer. The self-scaling strategy and hyperparameters are based on \cite{toscano2025kkans}. Note that for experiments using the quadratic potential ($\Phi(r)=r^2$), no smoothing was applied ($\phi=1.0$).}\label{tab:ac_first_order}
\end{table}

\subsubsection{Second-Order Optimization}
For the experiments involving a second-order optimizer for both the Allen-Cahn and Burgers' equations, we followed the methodology presented in \cite{urban2024unveiling}. The training begins with 5,000 Adam iterations for robust initialization, after which we switch to the SSBroyden optimizer for the remainder of the training. In this setup, vRBA is applied as an importance sampling strategy, where the collocation points are resampled every 100 iterations. The relevant hyperparameters are detailed in Table \ref{tab:second_order}.
\begin{table}[H]
\centering
\begin{tabular}{|l|c|c|}
\hline
\textbf{Hyperparameter} & \textbf{Allen-Cahn} & \textbf{Burgers} \\
\hline
N. of Adam training iterations & 5e3 & 5e3 \\
N. of SSBroyden training iterations & 6e4 & 6e4 \\
Number of hidden layers $N$ & 3 & 3 \\
Hidden layer dimension $H$ & 30 & 30 \\
Activation function & $\tanh(\cdot)$ & $\tanh(\cdot)$ \\
Periodicity Encoding & $\sin(\cdot), \cos(\cdot)$ & $\sin(\cdot), \cos(\cdot)$ \\
\hline
\multicolumn{3}{|c|}{\textbf{vRBA Parameters (Sampling)}} \\
\hline
Resampling Frequency & 100 iter. & 100 iter. \\
$\gamma$ (EMA memory) & 0.9 & 0.9 \\
$\eta$ (EMA learning rate) & 0.1 & 0.1 \\
$\phi$ (Smoothing) & 0.9 & 1.0 \\
\hline
\end{tabular}
\caption{Implementation details for the Allen-Cahn and Burgers' equations using a second-order SSBroyden optimizer, following the methodology in \cite{urban2024unveiling}.}
\label{tab:second_order}
\end{table}

\subsection{Operator Learning}
For the operator learning benchmarks, the architectures were chosen based on the specific model and task.

\subsubsection{DeepONet}
For the Bubble Growth Dynamics task, we employed a DeepONet architecture. The model consists of a branch network to process the input function and a trunk network to process the spatial/temporal coordinates. The implementation details, including the hybrid vRBA strategy, are summarized in Table \ref{tab:deeponet_full_details}.

\begin{table}[H]
\centering
\begin{tabular}{|l|l|c|}
\hline
\textbf{Category} & \textbf{Hyperparameter} & \textbf{Value} \\
\hline
\multirow{3}{*}{Branch Network} & Number of hidden layers & 4 \\
 & Hidden layer dimension & 100 \\
 & Activation function & GELU \\
\hline
\multirow{3}{*}{Trunk Network} & Number of hidden layers & 4 \\
 & Hidden layer dimension & 100 \\
 & Activation function & GELU \\
\hline
\multirow{2}{*}{Training Details} & Optimizer & Adam \\
 & Total N. of Parameters & 101,100 \\
\hline
\multirow{4}{*}{\shortstack{vRBA Parameters \\ (Hybrid Strategy)}} & Method & \shortstack{Weighting (Spatial Domain) \\ \& Sampling (Function Space)} \\
 & $\gamma$ (EMA memory) & 0.999 \\
 & $\eta$ (EMA learning rate) & 0.01 \\
 & $\phi$ (Smoothing) & 0.8 \\
\hline
\end{tabular}
\caption{Implementation details for the DeepONet used for the Bubble Growth Dynamics benchmark. The vRBA hyperparameters are consistent with those used in the first-order PINN experiments. Note that for experiments using the quadratic potential ($\Phi(r)=r^2$), no smoothing was applied ($\phi=1.0$).}\label{tab:deeponet_full_details}
\end{table}

\subsubsection{FNO and TC-UNet}
For the more complex Fourier Neural Operator (FNO) and Time-Conditioned U-Net (TC-UNet) models, we adopted the specific architectures and code provided in the reference TC-UNet study \cite{ovadia2025real}. This approach ensures our results are directly comparable to established benchmarks for these models. Consistent with the DeepONet experiment, we applied the same hybrid vRBA strategy: importance weighting was used for the spatial domain, while importance sampling was applied to the function space.

\subsection{JAX Implementation of the SSBroyden Optimizer}
\label{sec:appendix_jax_impl}

This appendix details our custom JAX implementation of the Self-Scaled Broyden (SSBroyden) optimizer, which was used for all second-order optimization experiments. The original method, proposed by Urbán et al.~\cite{urban2024unveiling}, relies on modified SciPy routines that are CPU-bound and not directly portable to a JAX-native, GPU-accelerated workflow.

Our implementation preserves the core SSBroyden update logic, which dynamically computes scaling ($\tau_k$) and updating ($\phi_k$) parameters. However, the line search portion of the algorithm required a complete rewrite. Due to the absence of SciPy's advanced line search routines in JAX Scipy, we developed a custom three-stage fallback line search mechanism to ensure robust convergence. This procedure creates a cascade of attempts with progressively less strict Wolfe conditions, starting with strict parameters ($c_2=0.9$) and relaxing them ($c_2=0.8$, then $c_2=0.5$) only upon failure. This adaptation was critical for ensuring the optimizer could consistently make progress on the challenging loss landscapes of the problems studied.

\end{document}